\documentclass[nohyperref]{article}

\usepackage{microtype}
\usepackage{graphicx}
\usepackage{subfigure}
\usepackage{booktabs} 

\usepackage{hyperref}

\usepackage[accepted]{icml2022}

\usepackage{amsmath}
\usepackage{amssymb}
\usepackage{mathtools}
\usepackage{amsthm}
\usepackage{mathmacros}

\usepackage[capitalize,noabbrev]{cleveref}

\theoremstyle{plain}
\newtheorem{theorem}{Theorem}[section]

\newtheorem{lemma}[theorem]{Lemma}

\theoremstyle{definition}
\newtheorem{definition}[theorem]{Definition}

\theoremstyle{remark}

\usepackage[textsize=tiny]{todonotes}

\newcommand{\z}{\m{z}}

\newcommand{\W}{\m{W}}
\newcommand{\x}{\m{x}}
\newcommand{\y}{\m{y}}
\newcommand{\N}{N}
\newcommand{\Id}{\m{I}}
\newcommand{\V}{V}
\renewcommand{\v}{\m{v}}

\newcommand{\ntk}{\Theta}

\newcommand{\h}{\m{h}}

\newcommand{\DD}{\m{D}}

\renewcommand{\th}{\sm{\theta}}

\newcommand{\M}{\m{M}}
\newcommand{\s}{s}
\newcommand{\X}{\m{X}}
\newcommand{\Y}{\m{Y}}
\newcommand{\Z}{\m{Z}}
\newcommand{\B}{\m{B}}
\newcommand{\G}{\m{G}}
\newcommand{\condtr}{\varphi}
\newcommand{\covtr}{\sigma}
\newcommand{\fourmom}[1][\z^*]{\sigma^{2}_{#1\cdot#1}}
\newcommand{\sigh}{\sigma^2_{\h^*}}
\newcommand{\sigphi}{\sigma^2_{\phi(\h^*)}}
\newcommand{\sigx}{\sigma_{\x}}
\newcommand{\tG}{M}

\newif\ifcomments
\commentsfalse 

\ifcomments
\newcommand{\xlc}[1]{{\color{blue}[XLC: #1]}}
\newcommand{\sss}[1]{{\color{purple}[SSS: #1]}}
\newcommand{\aga}[1]{{\color{red}[AA: #1]}}
\else
\newcommand{\xlc}[1]{}
\newcommand{\sss}[1]{}
\newcommand{\aga}[1]{}
\fi
\icmltitlerunning{DEQ initialization}

\begin{document}

\twocolumn[
\icmltitle{Deep equilibrium networks are sensitive \\
           to initialization statistics}




\begin{icmlauthorlist}
\icmlauthor{Atish Agarwala}{goog}
\icmlauthor{Samuel S. Schoenholz}{goog}
\end{icmlauthorlist}

\icmlaffiliation{goog}{Google Research, Brain Team}

\icmlcorrespondingauthor{Atish Agarwala}{thetish@google.com}
\icmlcorrespondingauthor{Samuel S. Schoenholz}{schsam@google.com}

\icmlkeywords{Machine Learning, ICML}

\vskip 0.3in
]



\printAffiliationsAndNotice{} 

\begin{abstract}
Deep equilibrium networks (DEQs) are a promising way to construct models which trade off memory for compute.
However, theoretical understanding of these models is still lacking compared to traditional networks, in part because
of the repeated application of a single set of weights.
We show that DEQs are sensitive to the higher order statistics of the matrix families from which they are
initialized. In particular, initializing with orthogonal or symmetric matrices allows for
greater stability in training.
This gives us a practical prescription for initializations which allow for training with
a broader range of initial weight scales.
\end{abstract}

\section{Introduction}

Deep equilibrium networks (DEQs) are a network architecture which uses implicitly-defined
layers to get benefits of deep networks with a smaller memory footprint \cite{bai_deep_2019}. DEQs have shown
competitive performance on image tasks \cite{bai_multiscale_2020, gilton_deep_2021},
as well as
on language tasks with transformer DEQ models
\cite{bai_deep_2019, bai_stabilizing_2021}.

A DEQ layer
can be defined as follows: given an input $\x$, the output $\z^*$ of a DEQ layer is given
implicitly by
\begin{equation}
\z^* = f_{\th}(\x, \z^*)
\label{eq:deq_basic}
\end{equation}
for a function $f$ parameterized by $\th$.
One way to solve for $\z^*$ is via direct iteration;
in this case the DEQ can be interpreted as a deep network
with identical parameters for each layer (tied weights),
as opposed to the more traditional case of independent parameters
for each layer (untied weights).

Despite their implicit definition,
DEQs can be trained with backpropagation. Equation \ref{eq:deq_basic}
gives an implicit equation for any Vector-Jacobian Products (VJPs) of the DEQ layer which
can also be found using a fixed point solver.
Training DEQs requires careful consideration of both initialization
as well as the solver used to find the fixed points \cite{bai_deep_2019, bai_stabilizing_2021}. This gives
another way to trade off memory and compute, but brings
the additional complication of maintaining convergence
of solvers throughout learning. For example, most practical applications use small
initializations to maintain stability
\cite{bai_deep_2019}. Empirical approaches to stabilize learning
include regularization of the Jacobian of the DEQ map \cite{bai_stabilizing_2021}.

While deep networks with untied weights have many theoretically-motivated
initialization schemes
\cite{glorot_understanding_2010, he_delving_2015, xiao_dynamical_2018, schoenholz_deep_2017, martens_rapid_2021},
understanding initialization in implicitly defined networks is a relatively new field
\cite{elghaoui_implicit_2021, massaroli_dissecting_2020, winston_monotone_2021}.
Some progress has been made on understanding DEQs in the
wide-network limit by formally relating their properties to
wide networks in the untied weights setting \cite{feng_neural_2020}. However, in practical settings
DEQs are often trained in a regime where the Jacobians approach divergence \cite{bai_deep_2019, bai_stabilizing_2021} -
or equivalently, in regimes where the DEQ effectively has large depth.
This is precisely the regime where the wide network limit breaks down
\cite{hanin_finite_2019}.
 
In this paper we develop a theory of DEQs at initialization which allows us to
understand the differences between the DEQ setting and the traditional
untied weights setting. In particular, we show that the choice of initial matrix family
can lead to quantitatively different behavior even in the wide network limit.
Using the linear DEQ, whose output can be computed analytically, we prove the following:
\begin{itemize}
\item We compute the convergence properties of the fixed point for
random i.i.d. matrices.
\item We prove that the tied weights setting
is qualitatively and quantitatively different than the untied weights setting,
particularly when using non-i.i.d Gaussian
weight matrices.
\item We show that variance is reduced when initializing with orthogonal or Gaussian orthogonal ensemble
(GOE) matrices.
\end{itemize}
We then examine the convergence properties of a simple non-linear DEQ:
\begin{itemize}
\item We derive a self-consistent set of equations which determine the stability
of the fixed point.
\item We prove that the untied weights theory gives a good approximation
to convergence for orthogonal and randomly initialized matrices away from
the divergence threshold.
\item We show that in the symmetric case the untied weights theory doesn't hold.
\end{itemize}

We conclude by demonstrating a practical consequence of the theory:
networks initialized with
random orthogonal or symmetric weight matrices have more stable learning
which is likely to converge more quickly than with standard i.i.d. initialization.
This increased stability also allows for networks to be trained with a broader set of
initial weight scales. This opens a new avenue for DEQ initialization:
to optimizing the initialization family rather than just the initialization scale.

\section{Theory of linear DEQs}

\subsection{Fixed point properties}

We define a \emph{linear DEQ} $\z^*(\x)$ as
\begin{equation}
\z^* = \W\z^*+\x
\label{eq:linear_deq}
\end{equation}
where $\x$ is an $\N$-dimensional input vector and
$\W$ is an $\N\times\N$ matrix.
One way to find $\z^*$ is as the limit of the iterated map
\begin{equation}
\z_{t+1} = \W\z_{t}+\x.
\end{equation}
This is a linear neural network with
input injection at each layer, as well as
the same $\W$ shared across layers.
Alternatively, we can directly solve for $\z^*$ using the
implicit equation:
\begin{equation}
\z^* = (\Id-\W)^{-1}\x.
\end{equation}
While this solution exists for all $\W$ with $\Id-\W$ invertible, the iterated
map is unstable if
$\W$ has any eigenvalues with magnitude greater than or equal to
$1$, as can be seen by the power series
\begin{equation}
\z^* = \lim_{t\to\infty}\z_{t} = \lim_{t\to\infty}\sum_{k=0}^{t}\W^{k} \x.
\label{eq:linear_deq_series}
\end{equation}
A similar convergence criterion was described in \cite{gao_global_2022}.
Convergence can be guaranteed by selecting $\W$ from a parameterized
family of matrices with bounded spectral norm,
as seen in \cite{kawaguchi_theory_2020}.

Previous work has focused analysis of the
\emph{untied weights} DEQ \cite{feng_neural_2020}. The untied weights
DEQ is defined by the iterative equation
\begin{equation}
\z_{t+1} \equiv \W_{k} \z_{t}+\x
\end{equation}
where the $\W_{k}$ are independent and drawn
from the same distribution. In the untied
weights case, there is no limit $\z_{\infty}$; however, the elements $\z_{t}$ have a
common limiting distribution for large $t$
and large $\N$. We will compare the
normal linear DEQ (\emph{tied weights} case)
with the untied weights case throughout the
remainder of this section.

\subsubsection{Random Gaussian initialization}

We now return to the weight tied setting. For
i.i.d. $\W$ with Gaussian entrie, for fixed $t$ the elements of $\z_{t}$ converge in distribution
to Gaussians as $\N\to\infty$ with previously
derived statistics \cite{yang_tensor_2021, feng_neural_2020}.
The moments are identical to the untied weights
case, and, for $\z_{0} = 0$ are given by
\begin{equation}
\expect[(\z_{t})_{i}] = \x_{i},~\Var[(\z_{t})_{i}] = \frac{1}{\N}\x\cdot\x\sum_{k=1}^{t}\V^{t} 
\end{equation}
where the elements of $\W$ (and $\W_{t}$) have variances $\V/\N$.

It has been argued \cite{feng_neural_2020} that the behavior as $t\to\infty$ can be derived
from the  $\N\to\infty$ limit, by applying the tensor program (TP) framework from 
\cite{yang_tensor_2021} (taking the limit $\N\to\infty$), and then taking the limit
$t\to\infty$. However, the TP framework is well-defined only for programs with finite length.
In particular, there is a non-zero probability that the iterated map will not converge even
for $\V <1$ due to eigenvalue fluctuations.
This has practical consequences as we will show
in Section \ref{sec:output_var}.

Nevertheless, in the case of 
linear networks, we are able to prove that $\z^*$ does in fact converge to a Gaussian
in distribution in the limit of infinite depth
(Appendix \ref{app:inf_depth_width}), extending the finite-depth result of \cite{feng_neural_2020}.
The basic idea of the proof is that, with probability close to $1$,
$\z_{t}$ and $\z^*$ become arbitrarily close for
large enough $\N$ and $t$ if $\V < 1$. The convergence in distribution of $\z_{t}$ can then
be used to show convergence in distribution of $\z^*$. This gives us the moments
\begin{equation}
\expect[(\z^*)_{i}] = \x_{i},~\Var[(\z^*)_{i}] = \frac{1}{\N}\x\cdot\x\frac{\V}{1-\V}
\end{equation}
as conjectured by \cite{feng_neural_2020}.
We provide an alternate derivation of the moments
using
operator-valued free probability theory in
Appendix \ref{app:op_free_prob}.

These results require that $\W$ (or the
$\W_{k}$ in the untied weights case) have i.i.d. random entries. In the untied weights case,
for low order moments we can relax the Gaussianity assumption and instead just assume
that the $\W_{k}$ are drawn independently from a rotationally invariant ensemble
(distribution over matrices)
- extending the result of \cite{feng_neural_2020} to other matrix families.
Let $\tr$ define
the $\N$-normalized trace (average eigenvalue).
Then, if $\V = \tr[\W_{k}^{\tpose}\W_{k}] <1$,
$\lim_{t\to\infty}\Var[(\z_{t})_{i}] = \frac{1}{\N}\x\cdot\x\frac{\V}{1-\V}$
(Appendix \ref{app:two_mom}).

\subsubsection{Tied weights, non-Gaussian initialization}

However, in the tied weights case, even low order moments can depend on the choice of ensemble for $\W$. We will consider two alternative
ensembles:
\begin{itemize}
\item \emph{Orthogonal} - $\W = \sqrt{\V}\m{O}$
for $\m{O}$ random orthogonal (Haar distributed).
\item \emph{GOE} - Random symmetric matrix,
Gaussian entries with variance $\V/\N$ off-diagonal and $2\V/\N$ on-diagonal.
\end{itemize}

For DEQs, a basic calculation (Appendix \ref{app:two_mom})
gives us
\begin{equation}
\Var[\z^*_{i}] = \frac{1}{\N}\left(\expect[\z^*\cdot\z^*]-\expect[\z^*]\cdot\expect[\z^*]\right)
\end{equation}
The second term evaluates to $\x\cdot\x$. The first term can be
written as a power series:
\begin{equation}
\expect[\z^*\cdot\z^*] = \expect_{\W}\left[\tr\left[\left(\sum_{k=0}^{\infty}\W^{k}\right)^{\tpose}\left(\sum_{k=0}^{\infty}\W^{k}\right)\right]\right] \x\cdot\x
\label{eq:z_mag_eq}
\end{equation}
which converges for $\V$ less than some $\V_{c}$. In
the i.i.d. random case, $\V_{c} = 1$.

We evaluate the power series in Appendix
\ref{app:two_mom}. In the orthogonal case, the variance matches the i.i.d. random case,
with $\V_{c}  = 1$.
However the GOE case gives dramatically different
behavior.
We can
compute the variance using Equation \ref{eq:z_mag_eq}. We note that
$\W_{k}^{\tpose} = \W_{k}$, and that $\tr[\W^{k}] = C_{k}V^{k}$,
where $C_{k}$ is the $k$th Catalan number. Using generating series, we have
\begin{equation}
\Var[(\z^*_{i})^2] = \frac{1}{\sqrt{1-4\V}}-\frac{1-\sqrt{1-4\V}}{2\V}-1
\end{equation}

One interesting feature is that $\Var[(\z^*_{i})^2]$ diverges differently for symmetric $\W$
than for random and orthogonal $\W$.
The divergence happens at $\V_{c} = 1/4$ - which reflects the spectral radius of $2\sqrt{\V}$
of the semi-circular law.

The behavior near $\V_{c}$ is of interest as well. For $\delta\equiv 1-\V/\V_{c}$, for the GOE ensemble
the asymptotic behavior of the variance of $(\z^*_{i})^2$
is
given by $O(\delta^{-0.5})$ for $\delta\ll 1$, while for random and orthogonal the divergence is $O(\delta^{-1})$. As we will see,
the differences become more pronounced for higher order moments.

The difference in $\Var[(\z^*_{i})^2]$ is relevant for learning dynamics. 
This can be seen explicitly in the large width limit where
the neural tangent kernel (NTK) controls learning
\cite{jacot_neural_2018}. For the linear DEQ, the
NTK is given by
$\frac{1}{\N^2}\expect[\z^*\cdot\z^*]^2 \x\cdot\x'$ for an input
pair $(\x, \x')$ (Appendix \ref{app:deq_NTK}) - 
so the GOE and random ensembles give different functions in the wide network limit.

\subsection{Variability of outputs}

\label{sec:output_var}

In the limit of infinite width, the first and second moments
of $\z^*$ are often sufficient to characterize the output of
the DEQ.
We saw that the first and second
moments of $\z^*$ are identical for random and orthogonal DEQs with
tied weights, which match the statistics of the untied weights case.
However, the GOE ensemble had a different second moment.
This already suggests that different distributions for $\W$ have
different behavior.

In the wide but finite dimensional setting, the differences
between the matrix families become more stark. While the full
distribution for finite $\N$ is intractable, we can understand
some of the differences by understanding a particular $4$th moment
of $\z^*$. These differences will be reflected in the convergence
dynamics of $\z^*$, as well as the implicit bias of the DEQ.

We begin by defining a particular $4$th moment which we call the
\emph{length variance}.
Given a random input $\x$ and a random matrix $\M$ we have:
\begin{definition}
\label{def:four_mom}
For $\z = \M\x$, the \emph{length variance}
of $\z$ is defined as:
\begin{equation}
\fourmom[\z] = \frac{1}{\N}\expect_{\M}[\Var_{\x}[\z\cdot\z]]
\end{equation}
\end{definition}

In order to compute the length variance, we use the following lemma (proof in Appendix \ref{app:four_mom}):
\begin{lemma}
\label{lem:four_mom}
Let $\M$ be an $\N\times\N$-dimensional random matrix. Let $\x$ be an
$\N$-dimensional vector with i.i.d. Gaussian elements with $0$ mean and variance
$\sigma_{\x}^2$, and let $\z = \M\x$. Then we have:
\begin{equation}
\fourmom[\z] = \frac{2}{\N} \sigma_{\x}^4\expect_{\M}\left[\tr[(\M^{\tpose}\M)^2]\right]
\end{equation}
\end{lemma}

Without loss of generality, we assume that $\sigma^2_{\x} = 1$ for the remainder of this section.

\subsubsection{Untied weights case}

We first consider the case of untied weights, where we compute
$\sigma^{2}_{\infty}\equiv \lim_{t\to\infty} \sigma^{2}_{\z_{t}\cdot\z_{t}}$.
In order to define $\sigma^{2}_{\z_{t}\cdot\z_{t}}$ in terms
of Lemma \ref{lem:four_mom}, $\x$ is the random input, and
$\M = \sum_{k=0}^{t}\prod_{j=0}^{k}\W_{k}$.

In the random and GOE cases, a direct calculation
(Appendix \ref{app:four_mom}) shows that
in the limit of large $\N$
we have
\begin{equation}
\sigma^{2}_{\infty} =  \frac{2}{\N}\left(\frac{2}{(1-\V)^2}+\frac{1}{(1-\V^2)^2}-\frac{2}{1-\V^2}\right)
\end{equation}
Meanwhile, for orthogonal $\W_{k}$ we have
\begin{equation}
\sigma^{2}_{\infty} = \frac{2}{\N}\left(\frac{2}{(1-\V)^2}-\frac{1}{1-\V^2}\right)
\end{equation}

The two main features here are that $\sigma^2_{\infty}$ is well-defined in
the untied weights case, and for $\delta \equiv 1-\V$,
$\sigma^{2}_{\infty} = O(\delta^{-2})$.

\subsubsection{Tied weights case}

The tied weights case is different.
We want to compute
\begin{equation}
\fourmom = \frac{2}{\N}\expect_{\W}\left[\tr\left[((\Id-\W)^{-\tpose}(\Id-\W)^{-1})^{2} \right]\right]
\label{eq:fourmom_inv}
\end{equation}
We first attempt solution via the formal power series
\begin{equation}
\fourmom = \frac{2}{\N}\expect_{\W}\left[\tr\left[\sum_{j, k, l, m = 0}^{\infty}(\W^{j})^{\tpose}\W^{k}(\W^{l})^{\tpose}\W^{m}\right]\right]
\label{eq:fourmom_power}
\end{equation}

In the orthogonal weights case the power series converges for $\V<1$ (Appendix \ref{app:four_mom}):
\begin{equation}
\fourmom = \frac{2}{(1-\V)^{3}}-\frac{1}{(1-\V)^2}
\end{equation}
which scales as $O(\delta^{-3})$ for $\delta\equiv 1-\V$. We immediately see that there
is more variance in the tied weights setting.

However, for random and GOE matrices the power series in
Equation \ref{eq:fourmom_power} diverges with non-zero probability.
This is due to the previously mentioned fluctuations in the largest eigenvalues,
which mean that iteration can diverge
in the random and GOE cases even when $\V<\V_{c}$. Therefore we attempt to compute the right hand side of
Equation \ref{eq:fourmom_inv} directly in the large $\N$ limit.
This can be done using operator-valued free probability theory (Appendix \ref{app:op_free_prob}).

In the random
case we have
\begin{equation}
\lim_{\N\to\infty}\N\fourmom = \frac{2\V^2}{(1-\V)^4}+\frac{4\V}{(1-\V)^3}+\frac{2}{(1-\V)^2}
\end{equation}
in the limit of large $\N$. For $\delta \equiv 1-\V$ we get $O(\delta^{-4})$ for small $\delta$. In the GOE case,
we have
\begin{equation}
\lim_{\N\to\infty}\N\fourmom = \frac{1}{2\V}\left(\frac{1}{(1-4\V)^{-5/2}}-\frac{1}{(1-4\V)^{-3/2}}\right)
\end{equation}
which diverges as $O(\delta^{-2.5})$ for $\delta\equiv 1-4\V$.

\subsubsection{Divergence of the variance}

The most interesting feature of $\fourmom$ is its behavior as $\V$ approaches the critical threshold
$\V_{c}$, where the iterative solving fails. As $\delta\equiv 1-\V/\V_{c}$ approaches $0$, $\fourmom$ diverges. This 
reflects the
increased variability in the outputs - both within a single DEQ and between DEQ models.
This can lead to instability in learning.

In the untied weights
case the divergence goes as $O(\delta^{-2})$ irrespective of the matrix ensemble. However, in the tied weights case we have
different, and worse, behavior ($O(\delta^{-2.5})$ for GOE, $O(\delta^{-3})$ for
orthogonal, and $O(\delta^{-4})$ for i.i.d. random).
For the random and GOE, these behaviors are valid when $\N$ is large
enough that the fluctuations in the largest eigenvalues are small ($\N\gg\delta^{-2/3}$, when Tracy-Widom fluctuations
\cite{tracy_levelspacing_1994} are controlled); there will be even more variability when $\N \ll\delta^{-2/3}$.
This is why the infinite power series solution fails; the direct calculation involving $(\Id-\W)^{-1}$, as
$\N$ goes to infinity, is the equivalent of taking the limit $\lim_{t\to\infty}\lim_{\N\to\infty} \fourmom[\z_{t}]$.

We can confirm these results numerically by computing $\tr[[(\Id-\sqrt{\V}\W)^{\tpose}(\Id-\sqrt{\V}\W)]^{-2}]$ for various $\V$ for a fixed
$\W$, drawn separately for each of the distributions. We plot the trace against
$\delta$
(Figure \ref{fig:four_mom}).
\begin{figure}[h]
\centering
\includegraphics[width=0.8\linewidth]{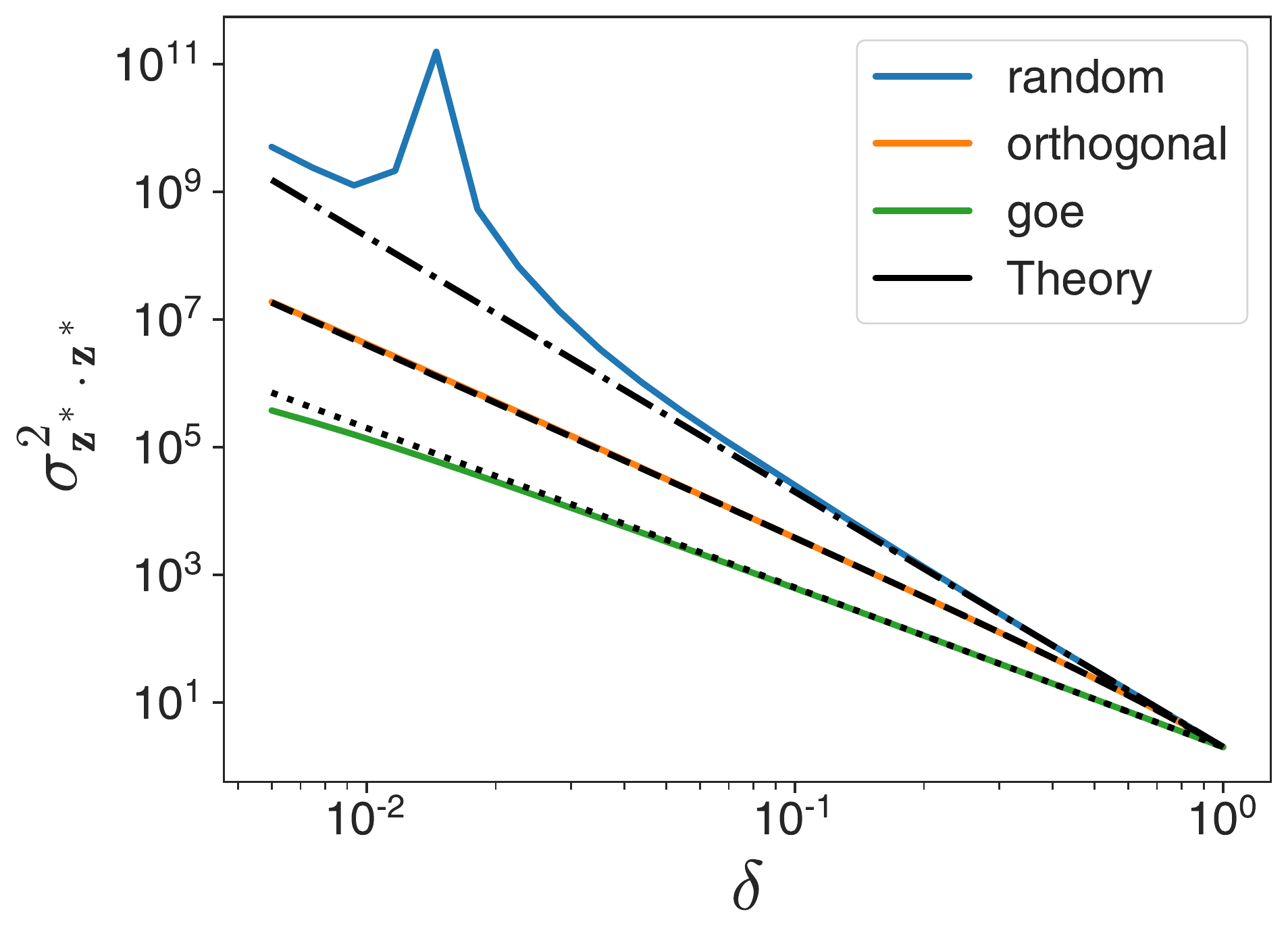}
\caption{ Empirical length variance $\fourmom$
computed using $\tr[[(\Id-\sqrt{\V}\W)^{\tpose}(\Id-\sqrt{\V}\W)]^{-2}]$
for $\N = 5000$. Plotted against $\delta = 1-\V/\V_{c}$ ($\V_{c} = 1/4$ for GOE and $1$ otherwise). Orthogonal is well predicted by theory, which also predicts intermediate
scale behavior for GOE and i.i.d. random. Random shows the most variance, including non-monotonic behavior due to
eigenvalue fluctuations.
}
\label{fig:four_mom}
\end{figure}

As expected, for the orthogonal $\W$
the trace follows its theoretical distribution due to self-averaging in large dimensions.
We also see that the GOE has a smaller trace than the orthogonal, but begins
to deviate from the infinite-width limit for small $\delta$. Finally we see that the
random case tends to have larger variability than either the random or orthogonal.
The non-monotonicity in this example is
associated with the eigenvalue with largest real part approaching $1$, even when $\V<\V_{c}$, due to finite-size
fluctuations.

This analysis suggests that in practice, the orthogonal and GOE ensembles
give more stable performance when initialized near the transition.
The orthogonal ensemble is guaranteed to converge when $\V <1$.
While the GOE ensemble
is also affected by fluctuations which can cause eigenvalues
larger than $1$, the smaller density of states at the edge of the
distribution means it is more stable to fluctuations than the random ensemble.

\section{Theory of non-linear DEQs}

\label{sec:nonlin_deq}

\subsection{Basic model and statistics}

Many of the ideas from the linear DEQ are reflected in the
non-linear setting as well. We consider the non-linear DEQ:
\begin{equation}
\z^* = \phi(\W\z^*)+\x
\label{eq:non-lin-deq-basic}
\end{equation}
where $\phi$ is an elementwise non-linearity. In order to understand the
statistics in the wide network limit, it is useful to instead study
$\h^* = \W\z^*$ defined by
\begin{equation}
\h^* = \W\phi(\h^*)+\W\x
\label{eq:non-lin-deq-h}
\end{equation}
As with the linear DEQ, $\h^*$ can be found as the limit of the iterated map
\begin{equation}
\h_{t+1} = \W\phi(\h_{t})+\W\x
\label{eq:h_iter}
\end{equation}

For random $\W$, the iterated map
implied by Equation \ref{eq:h_iter}
limits to a Gaussian process for a fixed number
of iterations in the large $\N$ limit \cite{feng_neural_2020, yang_tensor_2021}.
It remains an open question to show that
as the number of iterations increases, there is a limiting GP whose properties can be
computed by solving a fixed point equation for the kernel. If such a limiting kernel exists,
it can be solved for using the methods from \cite{feng_neural_2020}.

Following previous work on mean field theory of neural networks
\cite{poole_exponential_2016, schoenholz_deep_2017, feng_neural_2020, lee_wide_2019}, we study
$\sigh \equiv \expect_{\W}\left[\frac{1}{\N}\h^*\cdot\h^*\right]$ in the wide network limit.
In the GP limit, $\sigh = \Var[\h^*_{i}]$;
however, we can study $\sigh$ explicitly for large $\N$ even when the final limiting
distribution is unknown.

We note that $\sigh$ obeys the following self-consistent
equation:
\begin{equation}
\sigh = \expect_{\W}\left[\frac{1}{\N}(\phi(\h^*)+\x)^{\tpose}\W^{\tpose}\W(\phi(\h^*)+\x)\right]
\label{eq:sigh_eq}
\end{equation}
In order to solve for $\sigh$, we need to make some assumptions about $\h^*$, $\x$, and $\W$.
We can prove the following:

\begin{theorem}
\label{thm:sigh_thm}
Suppose $\phi(\h^*)$, $\x$, and $\W$ are freely independent, and
suppose that the distribution of $\W$ is rotationally symmetric.
Then we have:
\begin{equation}
\sigh = \V\left(\sigphi+2C_{\x, \phi(\h^*)}+\sigx^2\right)
\label{eq:sigh_free}
\end{equation}
where $\V = \expect[\tr[\W\W^{\tpose}]]$, $\sigx^2 \equiv \frac{1}{\N}\x\cdot\x$ and
\begin{equation}
\sigphi \equiv \expect[\phi(\h^*_{i})^2], ~ C_{\x, \phi(\h^*)} \equiv \left(\frac{1}{\N}\m{1}^{\tpose}\x\right) \expect[\phi(\h^*_{i})]
\end{equation}
If $\W$ is a random orthogonal matrix (times a fixed scale),
then Equation \ref{eq:sigh_free} holds as long as
$\phi(\h^*)$ and $\x$ are freely independent.
\end{theorem}
\xlc{
If $W$ is iid Gaussian, I think the free independence between $\phi(h^*) $ and $W$ can be dropped using a technique called Gaussian conditioning. Here is a sketch. Let $y=\phi(h^*)$. Let $a=Wy$. We can then decompose $W$ into a sum of $W=W^\perp + W_y$, where $W_y$ is a pseudoinverse $W_y=ay^{\dagger}$ and $W^\perp =\tilde W P_{y^\perp}$ and $\tilde W $ is a fresh independent copy of $W$; $P_{y^\perp}$ is orthogonal projection onto $y^\perp$.}
\aga{Noted, but will leave for later work!}
\begin{proof}
Using free independence of $\W$ with respect to the other variables we have
\begin{equation}
\sigh = \expect_{\W}\left[\tr\left[\W\W^{\tpose}\right]\right]\expect_{\W}\left[\frac{1}{\N}||\phi(\h^*)+\x||^{2}\right]
\label{eq:raw_sigh}
\end{equation}
which evaluates to
\begin{equation}
\sigh = \V\left(\expect_{\W}\left[\tr\left[ (\phi(\h^*)\phi(\h^*)^{\tpose}+2\phi(\h^*)\x^{\tpose}+\x\x^{\tpose}\right]\right]\right)
\label{eq:simplified_sigh}
\end{equation}
Using the independent of $\x$ and $\phi(\h^*)$, we arrive at
Equation \ref{eq:sigh_free}.
If $\W$ is a random orthogonal matrix, then Equation \ref{eq:simplified_sigh}
follows directly from Equation \ref{eq:sigh_eq}, and the rest of the
analysis holds.

\end{proof}

In Figure \ref{fig:h_var}, we iteratively solve for $\h^*$ for $\W$ using the different matrix families with
different scales. We then numerically solve Equation \ref{eq:sigh_free}. Comparing the
theoretical $\sigh$ with the empirical $\frac{1}{\N}\h^*\cdot\h^*$ (for $\N=1000$), we see
that for the random and orthogonal families, the empirical and theoretical quantities
are well matched.
The agreement in the orthogonal case
is not surprising, as the $\W^{\tpose}\W$
in Equation \ref{eq:raw_sigh} evaluates to the identity matrix.

\begin{figure}[h]
\centering
\includegraphics[width=0.8\linewidth]{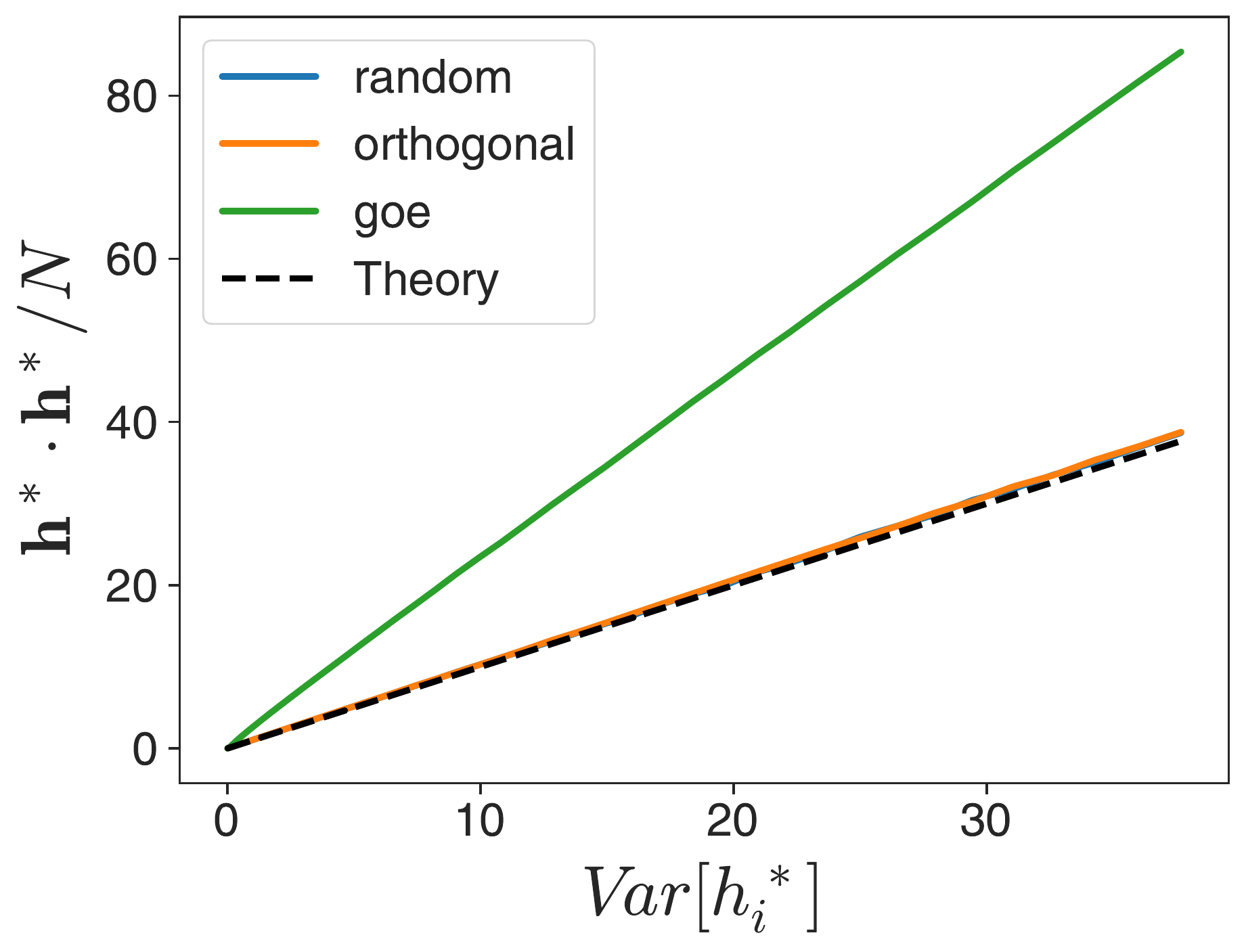}
\caption{Empirical $\frac{1}{\N}\h^*\cdot\h^*$ ($\N = 1000$)
versus theoretical $\Var[(\h_i^*)^2]$
for $\phi$ hard-$\tanh$. Random and orthogonal cases coincide and are well-predicted
by the assumption of $\h$ and $\W$ while GOE variance is larger.}
\label{fig:h_var}
\end{figure}

However, for the GOE, the actual variance is higher than the predicted one.
This suggests that $\phi(\h^*)$, $\x$, and $\W$ are not freely independent in the symmetric case.
In the linear case this was due to the
fact that the left and right singular vectors of
$\W$ are identical in the GOE case. Here
the presence of the non-linearity $\phi$
prevents the application of the linear theory,
but the difference between the GOE and
the other ensembles persists.

\subsection{Fixed point stability}

For non-linear DEQs, there is no general analytic solution for $\z^*$
or $\h^*$. However, one can still understand the stability of fixed
points of the iterative map of Equation \ref{eq:h_iter}.
Linearizing the dynamics in $\delta\h_{t}\equiv \h_{t}-\h^*$,
for small differences we have:
\begin{equation}
\delta\h_{t+1} = \W\circ\phi'(\h^*)\delta\h_{t}+\mathcal{O}(\delta\h^2)
\end{equation}
where $\circ \phi'(\h^*)$ represents the Hadamard product - 
multiplication by a diagonal matrix with entries $\phi'(\h^*)$.

The fixed point is stable under iteration if and only if
$\W\circ \phi'(\h^*)$ has eigenvalues with absolute value less than $1$.
We expect that the self-consistent Equation \ref{eq:sigh_eq}  to hold
near any fixed point. Indeed, Equation \ref{eq:sigh_eq} also describes
the situation where there is no stable fixed point but the statistics of
iteration converge. Therefore, by computing the empirical limiting value
of $\frac{1}{\N}\h^*\cdot\h^*$, we can predict
the existence of a stable fixed point.
If the statistics of $\h^*$ are consistent, then all such fixed points
will be stable.

We must posit a relationship between $\W$ and $\phi'(\h^*)$ to
compute the maximum eigenvalue.
For random and orthogonal $\W$, we have the following theorem:
\begin{theorem}
Let $\W$ be an $\N\times\N$ matrix and $\h^*$ be
an $\N\times 1$ dimensional vector. Let $\W$ and $\phi'(\h^*)$
be freely independent. If $\W$
is a random Gaussian matrix or a random orthogonal
matrix, as $\N\to\infty$ the spectral radius
$r$ of $\W\circ \phi'(\h^*)$ is
given by
\begin{equation}
r^2 = \expect_{\W}[\tr[\W^{\tpose}\W]]\expect\left[\frac{1}{\N}\phi'(\h^*)\cdot \phi'(\h^*)\right]
\end{equation}
\end{theorem}
\begin{proof}
In order to prove the theorem, we will prove that $\W\circ \phi'(\h^*)$ are
\emph{$R$-diagonal} operators - which have rotationally symmetric
empirical spectra in the complex plane with known maximum radius \cite{ mingo_free_2017}.
We will use the characterization from \cite{nica_diagonal_1996},
where an element $\m{R}$ is $R$-diagonal if it obeys a polar decomposition
\begin{equation}
\m{R} = \m{U}\m{P}
\end{equation}
where the sets $\{\m{U}, \m{U}^\dagger\}$ and $\{\m{P}, \m{P}^\dagger\}$ are freely independent, and
$\m{U}$ is Haar distributed. We note that $\m{R}$ has a Brown measure with support
given by an annulus in the complex plane centered at zero, with maximum radius
$r^2 = \expect[\tr[\m{R}^{\dagger}\m{R}]]$.

It remains to show that $\W\circ\phi'(\h^*)$ is $\m{R}$-diagonal. If $\W$ is orthogonal,
we already have a polar decomposition. If $\W$ is a Gaussian random matrix, we can use
the singular value decomposition $\W = \m{U}\sm{\sigma}\m{V}$. We have
$\m{U} = \m{U}$ and $\m{P} = \sm{\sigma}\m{V}\circ\phi'(\h^*)$. We immediately
have that $\{\m{U}, \m{U}^\dagger\}$ and $\{\m{P}, \m{P}^\dagger\}$ are
freely independent, and $\m{R}$ is $R$-diagonal.

Computing the trace of $\m{R}^{\tpose}\m{R}$ completes the proof.
\end{proof}

The GOE case is more complicated even when $\phi'(\h^*)$ and $\W$
are freely independent. Progress can be made if $\phi$ is monotonic. In this case,
$\phi'(\h^*)$ is non-negative, and $\W\circ\phi'(\h^*)$ has the same spectrum
as the symmetric matrix $\phi'(\h^*)^{1/2}\circ\W\circ\phi'(\h^*)^{1/2}$.
Free probability theory can then be used to numerically compute the
spectral radius of $\W\circ\phi'(\h^*)$
\cite{nica_diagonal_1996, mingo_free_2017}.

In certain cases, we can compute the spectral radius analytically.
If $\phi$ is the hard-tanh function, then the distribution is a linear combination
of a $\delta$-function at $0$ and a semicircular law (Appendix \ref{app:diag_W}). The spectral
radius is given by $2\sqrt{\V p(\h^*)}$, where $p(\h^*)$ is the probability
that $|\h^*_{i}|<1$ for $\h^*_{i}$ Gaussian distributed with some variance
$\sigh$ and mean $0$. Up to a factor of $2$ this is equivalent
to the radius for freely-independent orthogonal and random $\W$.

\begin{figure}[h]
\centering
\includegraphics[width=0.8\linewidth]{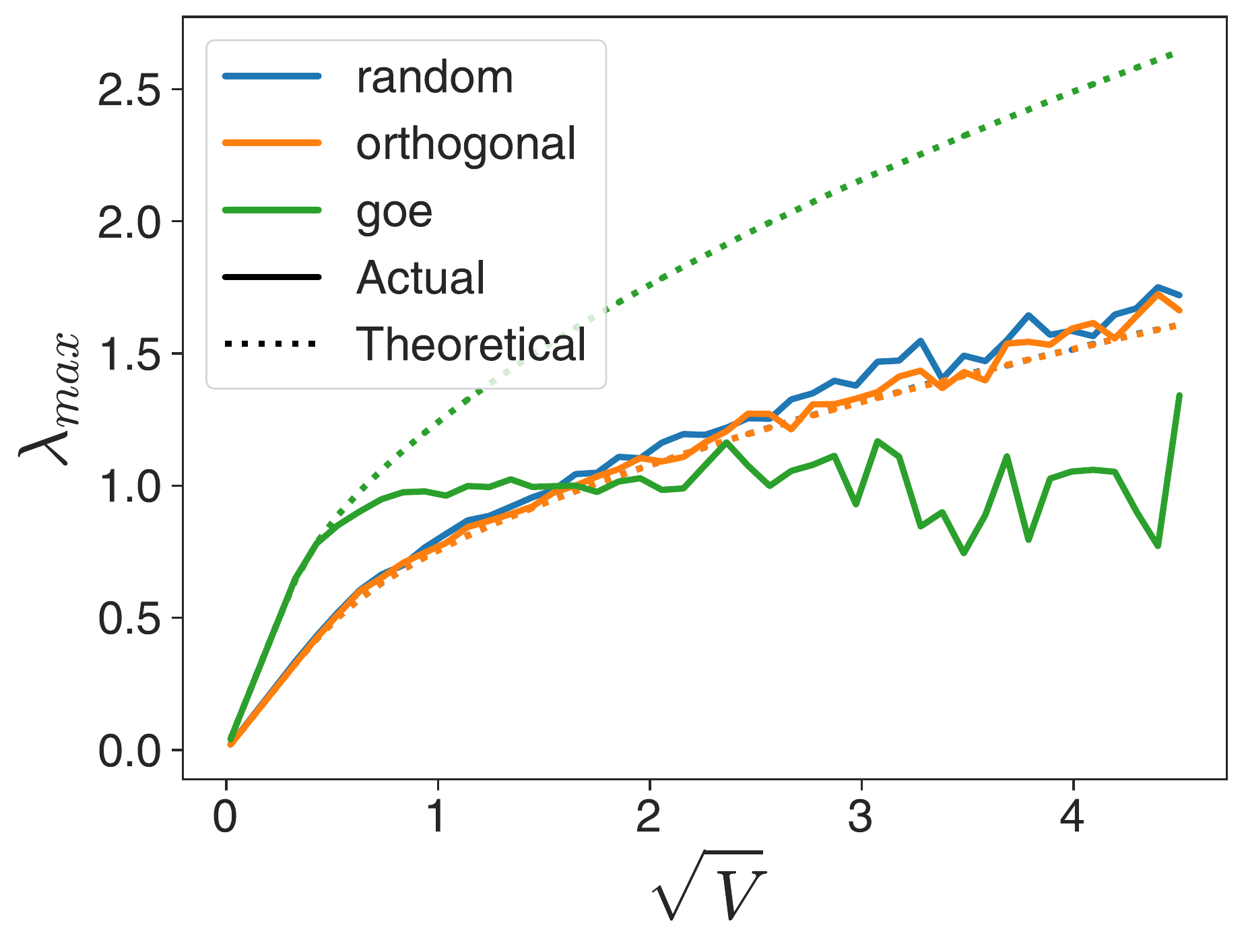}
\caption{Theoretical versus actual $\lambda_{max}$ for DEQ with hard-$\tanh$
non-linearity. Free-probability theory predicts random and orthogonal $\lambda_{max}$
for $\lambda_{max}\leq 1$. GOE is well predicted for small $\sqrt{\V}$ but not
near the transition.}
\label{fig:lamb_max}
\end{figure}

We can compare these theoretically-predicted stability conditions with
stability in practice. Given a randomly initialized DEQ, we can compute the
largest eigenvalue (by absolute value) and compare it to the theoretical
prediction for different values of $\sqrt{\V}$ (Figure \ref{fig:lamb_max}). We see that
for small $\V$, all matrix families are well-described by the theory. However,
in the symmetric case we see deviations at intermediate $\V$ as well.
As $\lambda_{max}$ approaches $1$ we see more deviations; finally,
when $\lambda_{max} > 1$ there are large fluctuations in the maximum eigenvalue.

For the random and orthogonal cases, the theoretical curves can be
used to predict the transition from a stable fixed point to no stable fixed point.
One way to detect convergence is to plot the residual
$||\h_{t+1}-\h_{t}||$ for some large $t$ after forward iteration of the DEQ map.
For all matrix families, this residual goes from near-zero to a non-zero
value as $\V$ increases (Figure \ref{fig:h_diff}, averaged over $1000$ samples).
We can predict the critical $\sqrt{\V}$ using the theoretical predictions for
$\lambda_{max}$. For the orthogonal and random cases, we see that
the transition is well-predicted by the theory, while in the symmetric case
the theory predicts a transition earlier than the actual transition.

\begin{figure}[h]
\centering
\includegraphics[width=0.8\linewidth]{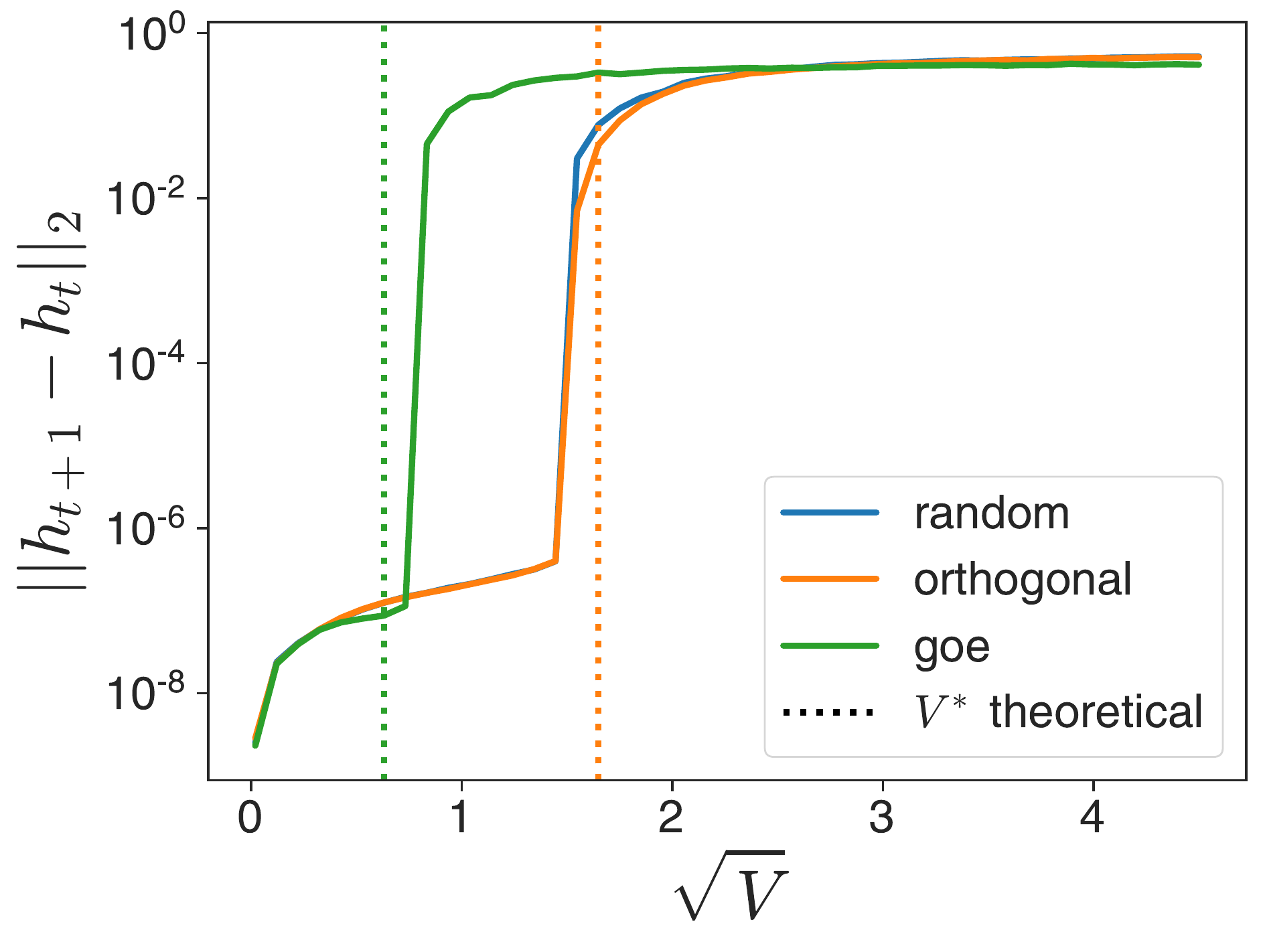}
\caption{$||\h_{t+1}-\h_{t}||$ for DEQs with hard-$\tanh$ non-linearity
($\N = 1000$, averaged over $1000$ samples). For large $\sqrt{\V}$, DEQ doesn't
converge. Divergence threshold is well predicted for random and
orthogonal, but more variance for GOE ensemble.}
\label{fig:h_diff}
\end{figure}

Even in the non-linear case the random and orthogonal
families have similar large-width behavior (which is predictive of some finite
width properties), but the symmetric matrix has different properties due
to its normality.

\section{Experiments}

The theoretical analysis has shown that the different matrix families
have different properties in both the wide network and finite width limits.
In particular, the orthogonal and GOE ensembles often have less
variability than the random ensemble. In this section,
we empirically explore the effects of initializing practical models
with different matrix families. We start by showing that orthogonal
initializations increase stability and performance for DEQs trained on MNIST.
We then show that for DEQ transformers,
using orthogonal or GOE ensembles increases the stability,
and sometimes even the speed of learning compared to i.i.d.
random matrices.

\subsection{Fully connected DEQ on MNIST}

We begin by training a fully-connected DEQ on MNIST
\cite{lecun_backpropagation_1989}. The network consists of a flattening
layer, followed by a ResNet DEQ layer with $\tanh$ non-linearity, and finally
a dense layer to project to $10$ logits.

We initialize the DEQ layer with random and orthogonal weight matrices
at different scales. Learning rate tuning of an ADAM optimizer with momentum
of $0.9$ suggested an optimal learning rate of $10^{-2}$ for all the conditions
studied. We then trained networks to convergence with $10$ random seeds for each
initialization family-scale pair.

For small scales,
the test error is similar for both types of initializations (Figure \ref{fig:mnist_error}).
However, for larger scales, the orthogonal initialization obtains lower test error.
The learning for the random initialization is less stable, as evidenced by a gap between the
median and mean test error across the random seeds. This suggests that the orthogonal
initialization increases the volume of hyperparameter space where DEQs are viable.

\begin{figure}[h]

\centering

\includegraphics[width=0.8\linewidth]{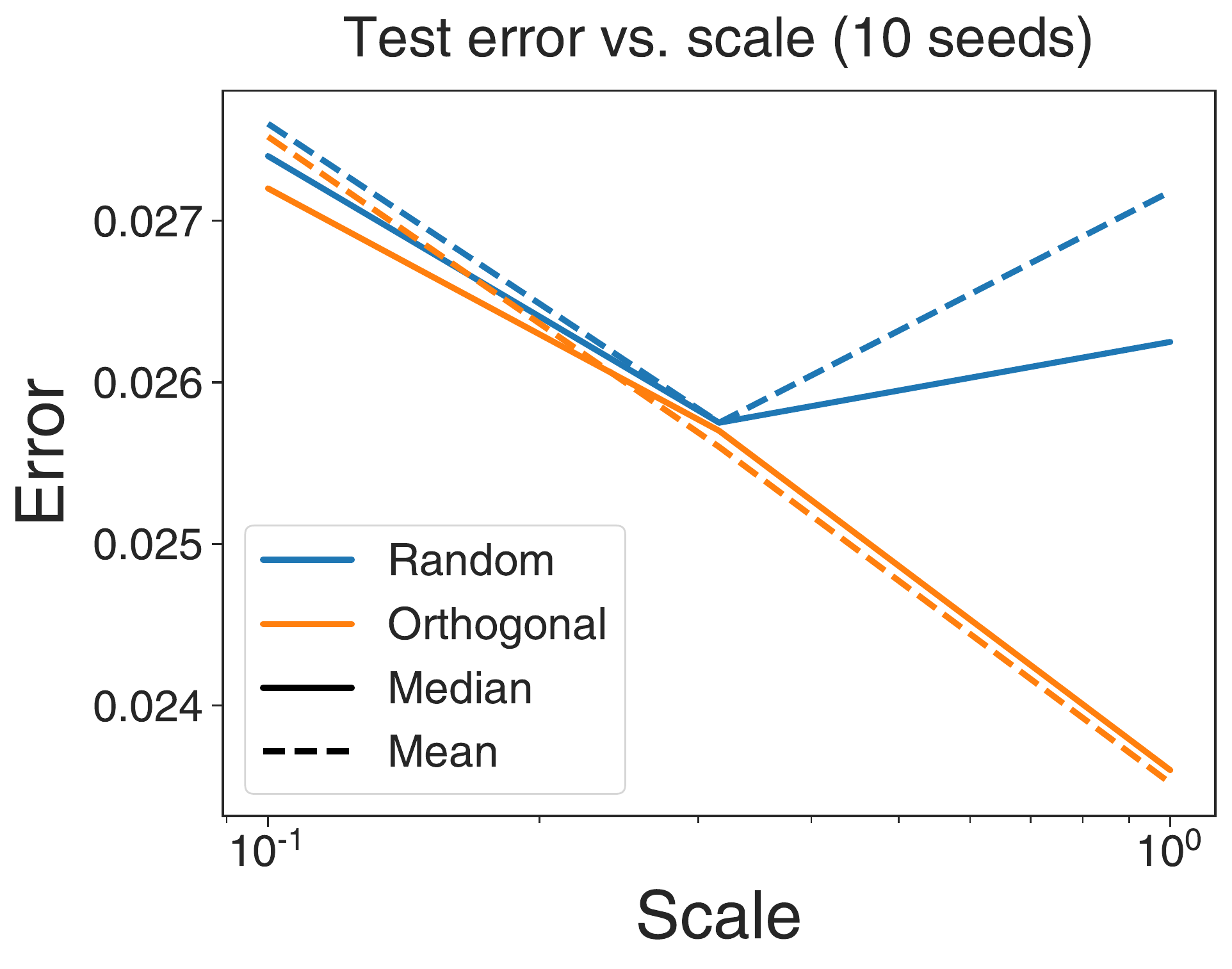}

\caption{Test error for fully connected DEQs trained on MNIST. Learning rate tuning was
performed for each initialization-scale pair, after which statistics were taken over 10 random seeds.
Orthogonal initialization allows for larger initialization scales to train stably, and achieves lower
test error.}

\label{fig:mnist_error}

\end{figure}

\subsection{DEQ transformer}

\begin{figure*}[t]

\centering

\begin{tabular}{ccc}
\includegraphics[width=0.3\linewidth]{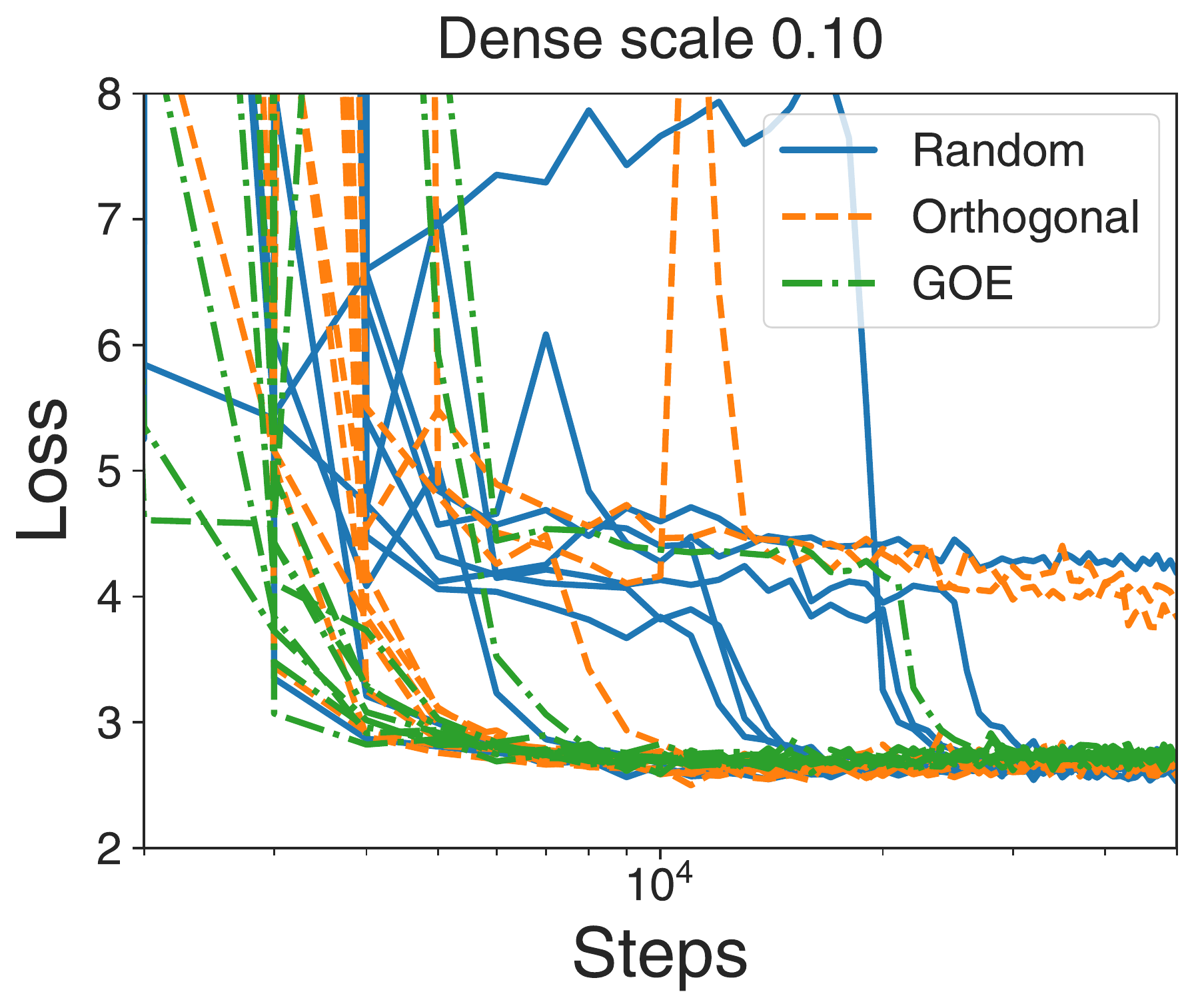} & \includegraphics[width=0.3\linewidth]{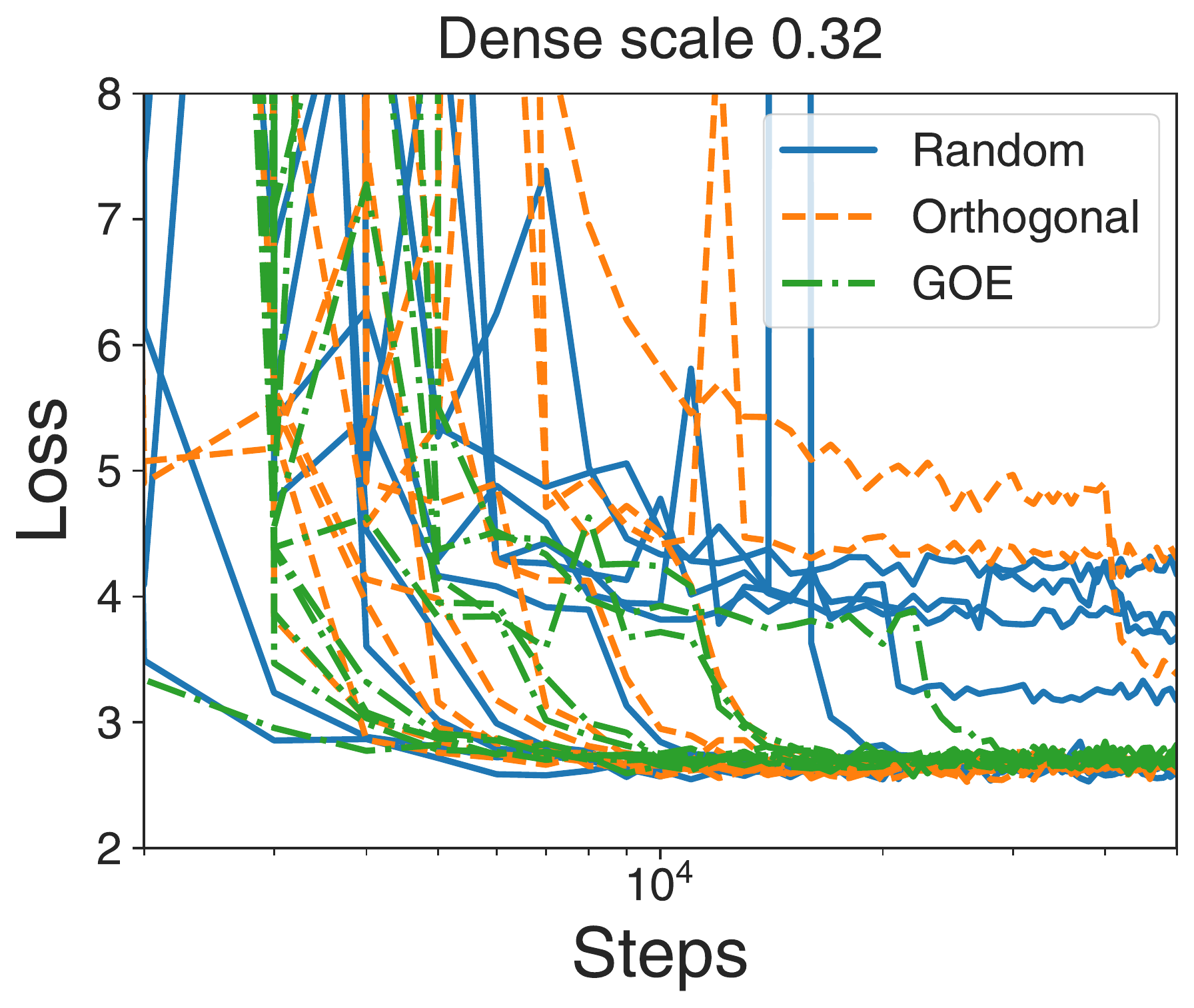} &
\includegraphics[width=0.3\linewidth]{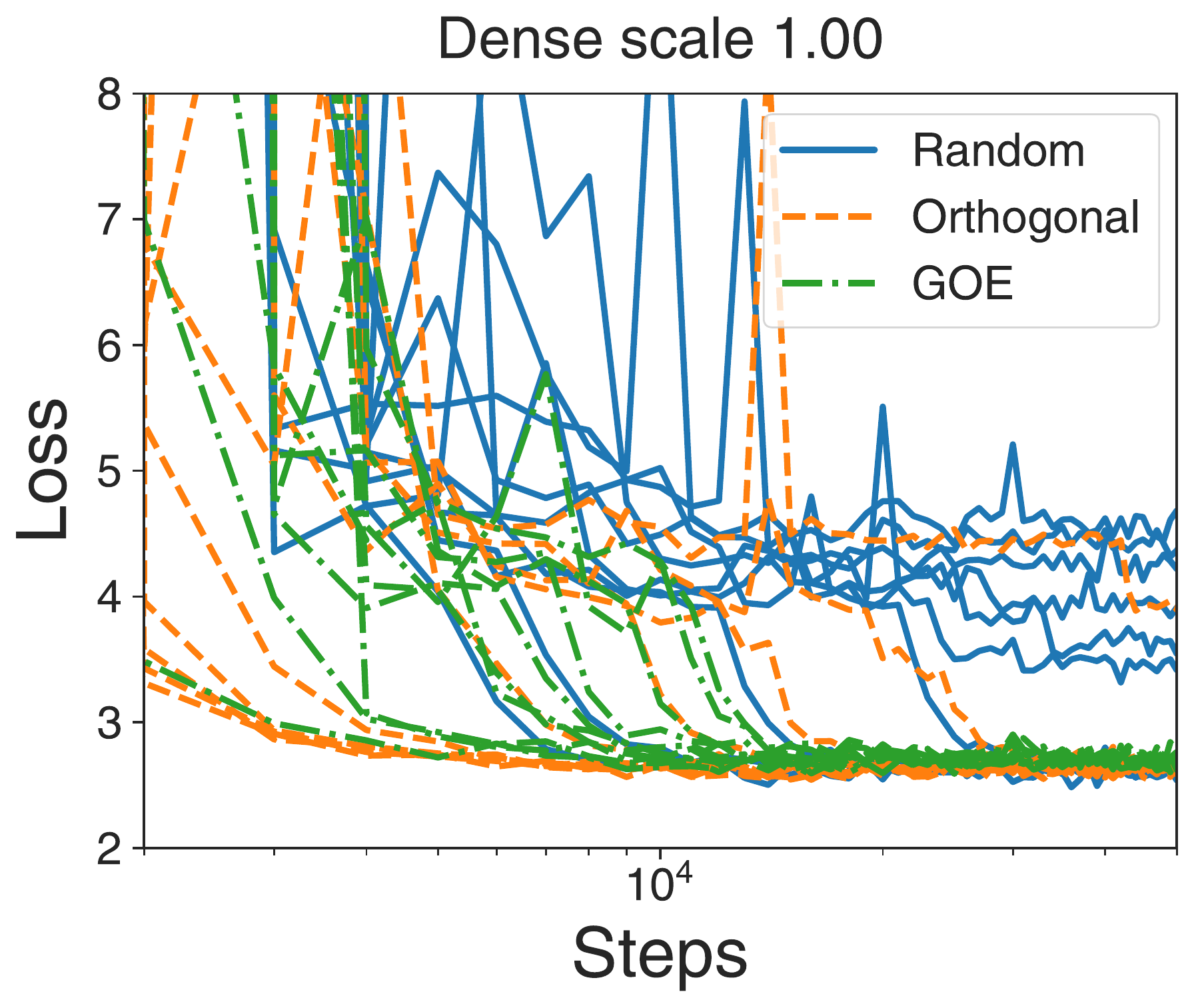}
\end{tabular}

\caption{Test loss for various $\sqrt{\V}$ and initial matrix ensembles, $10$ independent seeds. While random initializations
reach lowest test loss, they have a higher chance of diverging, and generally converge more slowly. GOE and orthogonal
initializations perform better as $\sqrt{\V}$ increases.} 

\label{fig:seed_perplexity}

\end{figure*}

We next examine the effects of the matrix ensembles
on a DEQ using a vanilla transformer layer from
\cite{al-rfou_characterlevel_2019} as
the base of the DEQ layer,  trained on Wikitext103 \cite{merity_pointer_2016}. With this architecture,
there are two main sets of matrices to initialize: the attention matrices and
the dense layers. We focused on the dense layers, as they have the same structure as the 
theoretical calculations.

We modified a Haiku implementation of the DEQ transformer \cite{khan_deq_2020}.
The details of our training procedure can be found in Appendix \ref{app:experiments}. We trained
on TPUv3.

\subsection{Stability and variability of learning}

We trained models with each of three matrix families (random, GOE, and random orthogonal),
with $\sqrt{\V} \in [10^{-1}, 10^{0}]$ over 10 random seeds. We find that the
average test loss is best for the GOE across all $\sqrt{\V}$, and orthogonal
is better than random at large $\sqrt{\V}$
(Table \ref{tab:test_perp}). However, we see that comparing
the best seeds from each family, the GOE performs worst.

\begin{table}[h]
\centering

\caption{Test perplexity on WIkitext103 for different $\sqrt{\V}$ and initial
matrix ensembles. Min and average taken over $10$ random seeds. Random initialization
has best performing models, but less stable learning than GOE and
orthogonal.}

\vskip 0.15in

\begin{sc}

\begin{tabular}{l|cccccc}
\toprule
 & \multicolumn{2}{c}{goe} & \multicolumn{2}{c}{orthogonal} & \multicolumn{2}{c}{random}\\
\midrule
$\sqrt{V}$ & Min & Ave & Min & Ave & Min & Ave \\
\midrule

$0.1$ & $68.1$ & $71.8$ & $60.7$ & $162.7$ & $56.8$ & $153.9$ \\
$0.3$ & $66.1$ & $69.8$ & $60.5$ & $173.4$ & $56.3$ & $224.9$ \\
$1.0$ & $66.3$ & $68.3$ & $57.4$ & $112.1$ & $55.6$ & $481.5$ \\

\bottomrule
\end{tabular}

\end{sc}

\label{tab:test_perp}

\end{table}

The discrepancy between the average and minimum
can be understood by
looking at the learning curves themselves.
For $\sqrt{\V} = 0.1$, where the random family performs best, we see that
many trajectories converge slowly to the equilibrium value (Figure \ref{fig:seed_perplexity}),
and some don't converge at all. In comparison, the orthogonal family
has trajectories which don't converge, but those that do converge more quickly.
All families from the GOE ensemble converge.

The difference between the families is more dramatic for $\sqrt{\V} = 1$.
The random initialization fails to converge to low test loss
very often compared to the orthogonal and GOE ensembles. This suggests that
a broader range of hyperparameters can be stably explored using
non-i.i.d. initializations, and that previously
reported limits of initialization
\cite{bai_deep_2019, bai_multiscale_2020, bai_stabilizing_2021}
may be overcome even without regularization.

This suggests that we can increase the stability of learning by switching to a different matrix family.
The GOE is the most stable, but also has the highest minimum perplexity. The orthogonal
family is a better choice, as it trades off less aggressively between performance and stability.

\section{Conclusions}

\subsection{Differences in the large-width limit}

Our theory and experiments suggested that even in the large width limit, symmetric matrices
behave differently from random and orthogonal matrices. Recent work has shown that
orthogonal and random matrices have similar behavior for finite depth and large
width \cite{huang_neural_2021, martens_validity_2021}; we conjecture
that the same may be true for DEQs.

In the linear DEQ case, we can understand the similarities and differences by noting that,
for random and orthogonal matrices,
$\tr[(\W^{k})^{\tpose}\W^{j}] = \delta_{jk}\tr[\W^{\tpose}\W]^{k}$. For i.i.d. random
matrices, this is due to the fact that the right and left singular vectors of $\W$
are independent; for the orthogonal family, this is due to the fact that the singular
values are all $1$.

However the symmetric basis has identical left and right singular vectors,
so the singular values of matrix powers are not determined solely by the average (squared)
singular value of $\W$ itself. We note that in the untied weights case, this is not an issue;
here instead of spectra of $\W^k$, we care about $\prod_{i=1}^{k}\W_{i}$ for independent
$\W_{i}$, which can be computed using the individual traces.

In general, matrices from \emph{normal} families will behave differently from random i.i.d.
matrices.
With any normal family, the iterated
DEQ map may not display free-independence
layer to layer.
As we saw in Section \ref{sec:nonlin_deq}, this is reflected in the non-linear case as well.

We also saw that for all families, the tied weights case had asymptotically more
finite-width deviations than the equivalent untied weights case as the networks
approached the divergence threshold. As many DEQs
are eventually trained near the unstable regime \cite{bai_deep_2019, bai_stabilizing_2021}, these differences
are relevant for understanding the performance of real DEQs.

\subsection{Practical implications}

DEQs suffer from training instability
\cite{bai_deep_2019, bai_stabilizing_2021}, perhaps more so than
``traditional" networks due to their (implicitly) iterative nature.
One approach is to use alternative parameterizations to induce
Lipschitz bounds with respect to parameters \cite{revay_lipschitz_2020},
or to guarantee convergence independent of parameter values
\cite{kawaguchi_theory_2020, winston_monotone_2021}. Another promising approach
involves
regularization of the Jacobian
($\W\circ\phi'(\h^*)$ in our simple case) to stabilize learning \cite{bai_stabilizing_2021}. 

Our theoretical work suggested a different way to to mitigate this instability
by optimizing over different matrix ensembles.
In the linear, untied weights case (normal deep network), learning dynamics with orthogonal
initializations has been extensively studied
\cite{saxe_exact_2014}, and orthogonal
initializations have been proposed as a method of reducing training instability \cite{schoenholz_deep_2017, hu_provable_2019, xiao_dynamical_2018}.

The orthogonal family provided clear benefits on MNIST, increasing the range
of good initialization scales, and eventually leading to better test performance.
For the vanilla DEQ transformer,
the alternative matrix ensembles provide benefits to training dynamics.
While our experiments didn't
show significant gains to the best performing networks, the increase in stability and
consequently typical training speed was significant. This gives a simple
way to improve training setups, and also seems to allow for a broader range of
hyperparameters
to be explored. With improved tuning, alternate matrix families may be able to
improve on best case performance instead of just average case.

\subsection{Future directions}

Given the success of the GOE ensemble, it may be useful to study
other normal ensembles. Random antisymmetric matrices have been
used to initialize RNN models \cite{chang_antisymmetricrnn_2018}. One can construct a Haar-distributed
random normal family with any eigenvalue distribution of interest,
complex or real.

Our analysis focused on fixed point solution via naive forward iteration,
while in practice (including in our own transformer experiments)
quasi-Newton methods like the Broyden solver with Anderson acceleration
are used to speed up convergence to the fixed point. It may be possible
to analyze these algorithms theoretically as well.

Another possible extension is applying similar methods to other
implicitly defined networks like neural ODEs \cite{chen_neural_2018}, which require
numerical solution of fixed points. Initializing with different matrix ensembles may improve
performance of these methods.


\bibliography{deq_icml_zotero, extra}
\bibliographystyle{icml2022}

\newpage
\appendix
\onecolumn

\section{Linear DEQ theory}

\label{app:linear_deq}

\subsection{Basic definitions}

\label{app:linear_def}

In this section we derive some statistics relevant to the linear DEQ,
defined implicitly by
\begin{equation}
\z^* = \W\z^*+\x
\end{equation}
for an $\N\times \N$ matrix $\W$ and an $\N\times 1$ dimensional vector $\x$.
This equation has solution
\begin{equation}
\z^* = (\Id-\W)^{-1}\x
\end{equation}
for $\W$ without eigenvalues equal to $1$.

We will usually consider $\W$ to be drawn from a random matrix ensemble.
The three ensembles we focus on will be the random Gaussian ensemble,
the Gaussian orthogonal ensemble of random symmetric matrices,
and the Haar-distributed orthogonal matrices. We will consider
$\W = \sqrt{\V}\W_{0}$, where $\W_{0}$ is drawn from these families
and $\V$ is a fixed factor controlling the average squared singular
value of $\W$.

We can think of $\z^*$ as the limit of the iterative map
\begin{equation}
\z_{t+1} = \W\z_{t}+\x
\end{equation}
This map converges if and only if the eigenvalues of $\W$ have magnitude less
than $1$.

We will also be interested in the \emph{untied weights} case where
the iterative equation is given by
\begin{equation}
\z_{t+1} = \W_{t}\z_{t}+\x
\end{equation}
where $\W_{t}$ are independent random matrices
from the same ensemble. Here $\z_{\infty} = \lim_{t\to\infty}\z_{t}$ does
not exist, but statistics like $||\z_{\infty}||^2 = \lim_{t\to\infty}\z_{t}\cdot\z_{t}$
do converge.

\subsection{Infinite depth, infinite width limit}

\label{app:inf_depth_width}

In \cite{feng_neural_2020}, the authors argue that DEQs under finite iteration meet the
conditions for the Tensor Programs (TP) framework
\cite{yang_tensor_2021} to apply, allowing them to use the untied weights calculation to compute
the finite-depth NTK. These results are then used to postulate a fixed point map for the
infinite-depth DEQ.

However, the theory in \cite{yang_tensor_2021} is valid only for programs of finite length. We
show how this limitation can be overcome in the case of linear DEQs with random initialization.
We will focus on the marginal distribution of
the $\z^*_{i}$ for a single input $\x$;
the covariance between pairs of inputs
$\x,~\m{y}$ follows naturally from our
arguments.

The key idea is that $\z_{t}$
better and better approximates $\z^*$, as
both $\N$ and $t$ increase. Since $\z_{t}$
converges to a Gaussian in $\N$, we can find
better and better approximations to $\z^*$
which are more and more Gaussian (with statistics
converging to the desired ones). By
carefully increasing $t$ and $\N$ together,
we show convergence to the desired distribution.

It is helpful to prove a Lemma about the
approximation scheme.
Let $\W$ have i.i.d. Gaussian elements with
variance $\V/\N$.
We begin by computing $||\z^*-\z_{t}||_{2}^{2}/\N$. We have:
\begin{equation}
\frac{1}{\N}||\z^*-\z_{t}||_{2}^{2} = \frac{1}{\N}\x^{\tpose}\left(\sum_{k=t+1}^{\infty}\W^{k}\right)^{\tpose}\left(\sum_{k'=t+1}^{\infty}\W^{k'}\right)\x
\end{equation}
Note that this power series representation is
not guaranteed to exist for finite $\N$;
fluctuations may take the largest eigenvalue
above $1$ even when $\V <1$. However, the 
empirical spectrum of $\W$ converges in probability to the circular law with radius
$\sqrt{\V}$; therefore given any $\delta > 0$
there exists some $\N_{1}$ such that the sum converges with probability
at least $1-\delta/2$ for all $\N > \N_{1}$. 

Having established the convergence of the
power series
representation, 
factoring gives us
\begin{equation}
\frac{1}{\N}||\z^*-\z_{t}||_{2}^{2} = \frac{1}{\N}(\z^*)^{\tpose}(\W^{t+1})^{\tpose}\W^{t+1} \z^*
\end{equation}

As $\N\to\infty$, the empirical distribution of $(\W^{t+1})^{\tpose}\W^{t+1}$ converges
in probability to a distribution with a maximum eigenvalue bounded by $t\V^{t+1}$ \cite{pennington_resurrecting_2017}.
Therefore, for every $\delta > 0$,
there exists an $\N_{2}>\N_{1}$ such that
for all $\N > \N_{2}$, we have
\begin{equation}
\frac{1}{\N}||\z^*-\z_{t}||_{2}^{2} \leq \frac{1}{\N}\z^*\cdot\z^* (t+1)\V^{t+1}
\end{equation}
with probability at least $1-\delta/2$. This means that we have the (loose) bound
\begin{equation}
||(\z^*)_{i}-(\z_{t})_{i}||_{2}^{2} \leq (\z^*\cdot\z^*) (t+1)\V^{t+1}
\end{equation}
with probably at least $1-\delta$.

We can use a similar argument to show that
$\z^*\cdot\z^*$ converges in probability to
$\frac{1}{1-\V} \x\cdot\x$ (limiting value computed in Appendix \ref{app:rand_inv}).
We have the bound
\begin{equation}
|(\z^*)_{i}-(\z_{t})_{i}|^{2} < \frac{2t}{1-\V}(\x\cdot\x)  \V^{t+1}
\label{eq:loose_bound}
\end{equation}
where for any $\delta > 0$ the bound holds with
probability at least $1-\delta$ for all $\N$
with $\N >\N_{3}$ for some $\N_{3}$.

We immediately see that for $\V<1$, the bound
decreases with $t$,
This gives us the following Lemma:
\begin{lemma}
\label{lem:lin_bound}
Let $F_{\N}(z)$ and $F_{\N}(z;t)$ be the
CDFs of $\z^*_{i}$ and $(\z_{t})_{i}$.
Fix $z$.
Then for $\epsilon > 0$, there exist
$\N_{c}$ and $t_{c}$ such that
\begin{equation}
|F_{\N}(z)-F_{\N}(z;t)| <\epsilon
\end{equation}
for $t>t_{c}$, $\N > \N_{c}$.
\end{lemma}
\begin{proof}
Let $1>\delta >0$, $1>\tilde{\epsilon}> 0$.
There exists a $t_{c}$, $\N_{c}$ such that
for $t>t_{c}$, $\N > \N_{c}$,
$|(\z^*)_{i}-(\z_{t})_{i}|<\tilde{\epsilon}$
for
$t>t_{c}$, $\N > \N_{c}$ with probability at
least $1-\delta$ (for example, by choosing
$t_{c} = O(\ln(\V/\tilde{\epsilon}))$.

This means that we have
\begin{equation}
P(\z^*_{i}<z)\geq (1-\delta)P((\z_{t})_{i}<z-\tilde{\epsilon})
\end{equation}
as well as
\begin{equation}
P(\z^*_{i}<z)\leq \frac{1}{1-\delta}P((\z_{t})_{i}<z+\tilde{\epsilon})
\end{equation}
Note that there exists a constant $B$ such that $F_{\N}(z;t)-F_{\N}(z';t)\leq B (z-z')$ for
$z-z' < 1$. Therefore we have:
\begin{equation}
P(\z^*_{i}<z)\geq (1-\delta)[P((\z_{t})_{i}<z)+B\tilde{\epsilon}]
\end{equation}
\begin{equation}
P(\z^*_{i}<z)\leq \frac{1}{1-\delta}[P((\z_{t})_{i}<z)+B\tilde{\epsilon}]
\end{equation}
Therefore we have the bound:
\begin{equation}
|F_{\N}(t,z)-F_{\N}(t,z)| < 2(\delta + B\tilde{\epsilon})
\label{eq:joint_bound}
\end{equation}

Given an $\epsilon>0$, choose $\delta$
and $\tilde{\epsilon}$ such that $(\delta + B\tilde{\epsilon}) < \epsilon$. Then the lemma
follows immediately from Equation \ref{eq:joint_bound} with
$t_{c}$ and $\N_{c}$ defined as above.
\end{proof}

Armed with this lemma, we can prove the 
following:
\begin{theorem}
Let $\W$ be i.i.d. Gaussian with elements of
variance $\V/\N$, and let $\z_{t}$ and $\z^*$ be defined as above. Then for $\V < 1$,
as $\N\to\infty$
each $\z^*_{i}$
converges in distribution to a Gaussian with
mean $\x_{i}$ and second moment
\begin{equation}
\expect[(\z^*_{i})^2] = \frac{1}{1-\V} \frac{1}{\N}\x\cdot\x
\end{equation}
\end{theorem}
\begin{proof}
We will show that $\lim_{\N\to\infty}F_{\N}(z)\to F(z)$, where $F_{\N}(z)$ are the CDFs of the
$(\z_{t})_{i}$ and $F(z)$ is the CDF of the Gaussian with appropriate parameters.
We first note that the $(\z_{t})_{i}$ converge
in distribution to a Gaussian with mean
$\x_{i}$ and second moment $m_{2}(t)$ given by
\begin{equation}
m_{2}(t) = \frac{1-\V^{t}}{1-\V}\frac{1}{\N}\x\cdot\x
\end{equation}
(as per \cite{feng_neural_2020, yang_tensor_2021}). Let $F(z;t)$ be the
corresponding limiting CDF.
We note that
\begin{equation}
|\Phi(x/\sigma_1)-\Phi(x/\sigma_2)| \leq \left[\max_{x'}\left.\frac{d\Phi(x'/\sigma)}{d\sigma}\right|_{\sigma=\sigma_1}\right] (\sigma_1-\sigma_2)+ O((\sigma_1-\sigma_2)^2)
\end{equation}
where $\Phi$ is the standard Gaussian CDF. For
fixed $\sigma$, the max derivative is bounded.
Therefore we can write:
\begin{equation}
|F(z;t)-F(z)|\leq a (\sigma_{t}-\sigma_{\infty})+b(\sigma_{t}-\sigma_{\infty})^{2}
\end{equation}
for some constants $a$ and $b$, where
\begin{equation}
\sigma_{t}^{2} = \left(\frac{V^{t}}{1-\V}-1\right) \frac{1}{\N}\x\cdot\x
\end{equation}
We can re-write the bound as:
\begin{equation}
|F(z;t)-F(z)|\leq A \V^{t}
\label{eq:inf_N_bound}
\end{equation}
for some fixed $A$.

Let $F_{N}(z; t)$ be the CDF of $(\z_{t})_{i}$.
We note that $\lim_{\N\to\infty} F_{\N}(z; t) = F(z ; t)$. We will show that $F_{\N}(z)$ can
be made arbitrarily close to $F_{\N}(z; t)$
for large $\N$ and $t$, allowing it to get
arbitrarily close to $F(z)$. This will complete
the proof.

Fix some $z$. Fix an $\epsilon>0$. From
Equation \ref{eq:inf_N_bound}, there exists a
$t_{b}$ such that $|F(z;t)-F(z)| <\epsilon/3$
for all $t > t_{b}$. From Lemma \ref{lem:lin_bound}, there is a $t_{c}$ and
an $\N_{c}$ such that
$|F_{\N}(z;t)-F_{\N}(z)| < \epsilon/3$.

Fix
some $t_{a}$ larger than $t_{b}$ and $t_{c}$.
Then there exists an $\N_{a}$, greater than
$\N_{b}$ and $\N_{c}$, such that
$|F_{\N}(z;t_{a})-F(z; t_{a})| < \epsilon/3$
for all $\N > \N_{a}$. Using the triangle
inequality, we have:
\begin{equation}
|F_{\N}(z)-F(z)|<\epsilon
\end{equation}
for all $\N > \N_{a}$. This completes the
proof.
\end{proof}

Similar arguments apply to compute the
covariate statistics of $\z^*(\x)$, $\z^*(\m{y})$ for different data points $\x$ and $\y$. This
exact method of proof will not suffice for the non-linear case where the closed form solution is
not known; however, in cases where the DEQ provably linearly converges to its fixed point,
a similar argument may hold (depending on the properties of the Hessian).

\subsection{Second moment}

\label{app:two_mom}

In the untied weights case, define
\begin{equation}
\Var[(\z_{\infty})_{i}^2] \equiv \lim_{t\to\infty }\Var[(\z_{t})_{i}^2]
\end{equation}
For rotationally invariant $\W_{t}$ we have
\begin{equation}
\Var[(\z_{\infty})_{i}^2] =\lim_{t\to\infty }\Var[(\z^*_{i})^2] = \frac{1}{\N}\left(\expect[\z_{t}\cdot\z_{t}]-\expect[\z_{t}]\cdot\expect[\z_{t}]\right)
\end{equation}
The latter gives $\x\cdot\x$. The first term
can be computed recursively:
\begin{equation}
\expect[\z_{t+1}\cdot\z_{t+1}] = \expect\left[(\W_{t}\z_{t}+\x)^{\tpose}(\W_{t}\z_{t}+\x)\right] = \x\cdot\x+\tr[\W_{t}^{\tpose}\W_{t}]\expect[\z_{t}\cdot\z_{t}]
\end{equation}
Therefore we obtain the limit:
\begin{equation}
\Var[(\z_{\infty})_{i}^2] = \frac{\V}{1-\V} \left(\frac{1}{\N}\x\cdot\x\right)
\end{equation}
valid for $\V = \tr[\W_{t}^{\tpose}\W_{t}] <1$.

Now consider the quantity
$\Var[(\z^*_{i})^2]$ in the tied weights case,
averaged over the ensemble of $\W$.
If $\W$ is from a rotationally invariant
family, we have $\Var[(\z^*_{i})^2] = \frac{1}{\N}\left(\expect[\z^*\cdot\z^*]-\expect[\z^*]\cdot\expect[\z^*]\right)$. For all rotationally
invariant ensembles, we have $\expect[\z^*]\cdot\expect[\z^*] = \x\cdot\x $. The first term can be computed as:
We have:
\begin{equation}
\expect[\z^*\cdot\z^*] = \expect\left[\x^{\tpose}(\Id-\W)^{-\tpose}(\Id-\W)^{-1}\x\right] = \expect[\tr[(\Id-\W)^{-\tpose}(\Id-\W)^{-1}]] \x\cdot\x
\end{equation}
If the spectral radius of $\W$ is less than $1$,
we can write the power series
\begin{equation}
\frac{1}{\N}\expect[\z^*\cdot\z^*] = \expect\left[\tr\left[\sum_{j, k=0}^{\infty}(\W^{j})^{\tpose}\W^k \right]\right]
\end{equation}
For random and orthogonal $\W$, only terms with $j =k$ contribute and we have
\begin{equation}
\frac{1}{\N}\expect[\z^*\cdot\z^*] = \sum_{k=0}^{\infty} \V^{k} = \frac{1}{1-\V}
\end{equation}
and we get $\Var[(\z^*)_{i}^2] = \frac{\V}{1-\V}$
as in the untied weights case.

However, for random symmetric matrices, all terms with $j+k$ even contribute.
he semi-circular law gives us:
\begin{equation}
\tr[\W^ {2k}] = C_{k}\V^{k},~C_{k} = \frac{1}{k+1}\binom{2k}{k}
\end{equation}
where $C_{k}$ is the $k$th Catalan number.
The trace vanishes for odd $k$. We also note that the Catalan numbers have the generating
function $f_{c}(x)$ where
\begin{equation}
f_{c}(x) \equiv \sum_{k} C_{k} x^k = \frac{1-\sqrt{1-4x}}{2x}
\end{equation}

Computing, we have:
\begin{equation}
\expect[(\z^*_{i})^2] =  \tr\left[\sum_{j, k = 0}^{\infty}\W^{j}\W^{k}\right] = \sum_{i}(2i+1)\V^{i}C_{i}
\end{equation}
Using the theory of generating series, we have
\begin{equation}
\expect[(\z^*_{i})^2] =  2\V f'_{c}(\V)+f_{c}(\V)
\end{equation}
\begin{equation}
\expect[(\z^*_{i})^2] =  2\V \left(-\frac{1-\sqrt{1-4\V}}{2\V^2}+\frac{1}{2\V\sqrt{1-4\V}}\right)+\frac{1-\sqrt{1-4\V}}{2\V}
\end{equation}
\begin{equation}
\expect[(\z^*_{i})^2] = -\frac{1-\sqrt{1-4\V}}{2\V}+\frac{1}{\sqrt{1-4\V}}
\end{equation}

Therefore GOE distributed matrices have different statistics than the untied weights
case, as well as the random and orthogonal DEQs. In particular, the diverging behavior
is different. Here the divergence happens at $\sqrt{\V} = \frac{1}{2}$ due to the radius of the semicircular
law. In addition, if we write $\sqrt{\V} = 1/2-\delta$, for $\delta \ll 1$, the second moment is
$O(\delta^{-1/2})$, compared to $O(\delta^{-1})$ for the other cases.

\subsection{Fourth moment}

\label{app:four_mom}

The variability of DEQ outputs can be understood by computing the \emph{length variance}.
For $\z = \M\x$, with $\x$ i.i.d. Gaussian with $0$ mean and variance $\sigma^2_{\x}$, we have
\begin{equation}
\fourmom[\z] = \frac{1}{\N}\expect_{\M}[\Var_{\x}[\z\cdot\z]]
\end{equation}
which by Lemma \ref{lem:four_mom} is equivalent to
\begin{equation}
\fourmom[\z] = \frac{1}{\N}\expect_{\M}[\tr((\M^{\tpose}\M)^2)]\sigma_{\x}^4
\end{equation}
We give a proof of the lemma below.

\begin{proof}
Let $\{ s\}$ be the singular values of $\M$, with
associated singular vectors $\v_{\s}$.
We have:
\begin{equation}
\z\cdot\z = \sum_{\s} \s^2 |\x\cdot\v_{\s}|^{2}
\end{equation}
Taking expectation over $\x$ we have
\begin{equation}
\expect_{\x}[\z\cdot\z]^2 = \sigma_{\x}^4 \sum_{\s, \s'}\s^{2}(\s')^2 
\end{equation}
We also have
\begin{equation}
\begin{split}
(\z\cdot\z)^2 & = \left(\sum_{\s} |\s|^{2} |\x\cdot\v_{\s}|^{2}\right)^2 \\
& = \sum_{\s}|\s|^4|\x\cdot\v_{\s}|^4 +\sum_{\s\neq\s'} |\s'|^{2}|\s|^{2} |\x\cdot\v_{\s}|^{2}|\x\cdot\v_{\s'}|^{2}
\end{split}
\end{equation}
Note that $\expect_{\x}[|\x\cdot\v_{s}|^4] = 3\sigma_{\x}^{4}$. This gives us:
\begin{equation}
\fourmom[\z] = \frac{2}{\N}\sigma_{\x}^{4}\expect_{\M}\left[\sum_{s}s^4\right] = \frac{2}{\N}\sigma_{\x}^{4} \expect_{\M}\left[\tr[(\M^{\tpose}\M)^2]\right]
\end{equation}
\end{proof}

We will assume $\sigma_{\x}^2 = 1$ for the remainder of this section.

\subsubsection{Untied weights}

In the untied weights case, we have
\begin{equation}
\sigma^{2}_{\infty} = \frac{2}{\N}\expect\left[\tr\left[\sum_{j,k,l,m = 0}^{\infty} \M_{j}^{\tpose}\M_{k}\M_{l}^{\tpose}\M_{m} \right]\right]
\end{equation}
where $\M_{k}\equiv \prod_{t=1}^{k}\W_{t}$.
The terms the terms that contribute are $j = k$, $l = m$ or
$m = j$, $l = k$. If $j\neq l$, the two factors are freely independent and
we can factor the trace. However, in the case $j = k = l = m$, the squared
eigenvalues of $\M_{j}^{\tpose}\M_{j}$ contribute. This gives us:
\begin{equation}
\sigma^{2}_{\infty} = \frac{2}{\N}\sum_{i=0}^{\infty} 2(i+1)\V^{i} - 2\V^{2i}+B_{i}\V^{2i}
\end{equation}
where $B_{i} = \expect[\tr[(\M_{i}^{\tpose}\M_{i})^2]]$. This gives us:
\begin{equation}
\sigma^{2}_{\infty} = \frac{2}{\N}\left(\frac{2}{(1-\V)^2}-\frac{2}{1-\V^2} + \sum_{i=0}^{\infty}B_{i}\V^{2i}\right)
\end{equation}
In the random and GOE cases, as $\N\to\infty$ $B_{k} = k+1$. This gives us
\begin{equation}
\lim_{\N\to\infty}\frac{\N}{2}\sigma^{2}_{\infty} = \frac{1}{(1-\V^2)^2}+\frac{2}{(1-\V)^2}-\frac{2}{1-\V^2}
\end{equation}
In the orthogonal case, $B_{k} = 1$. This gives us
\begin{equation}
\sigma^{2}_{\infty}  = \frac{2}{\N}\left(\frac{2}{(1-\V)^2}-\frac{1}{1-\V^{2}}\right)
\end{equation}
This is smaller than the GOE and random cases, but still has the same asymptotics as $\V\to 1$.

Therefore in the untied weights case, $\sigma^2_{\infty}$ scales as $(1-\V)^{-2}$ in all cases.
For the orthogonal case we have an analytic form for finite $\N$, and for the random and GOE
we have limiting behavior for $\N\to\infty$.
Evidently, as long as $\V <1$ we expect convergence of the statistics in the large
$\N$ limit.

\subsubsection{Tied weights}

Analogously to the untied weights case, we can attempt a power series solution:
\begin{equation}
\fourmom[\z^*] = \frac{2}{\N}\expect\left[\tr\left[\sum_{j,k,l,m = 0}^{\infty} (\W^{j})^{\tpose}(\W^{k})(\W^{l})^{\tpose}(\W^{m}) \right]\right]
\end{equation}

This sum can be carried out in the orthogonal case.
Only terms with $k+m = j+l$ contribute. Each side of the equation is independent, so for $k+m = i$ there are
$(i+1)^2$ possible combinations. We have:
\begin{equation}
\fourmom[\z^*] = \frac{2}{\N}\sum_{i=0}^{\infty} (i+1)^2\V^{i} = \sum_{i=0}^{\infty} (i+1)(i+2)\V^{i}-(i+1)\V^{i} = \frac{2}{\N}\left(\frac{2}{(1-\V)^{3}}-\frac{1}{(1-\V)^2}\right)
\end{equation}
Here we see that the divergence is $(1-\V)^{-3}$ - asymptotically different from the untied weights case.

For the random and GOE cases, the power series formulation diverges for finite $\N$. This is
due to the fluctuations in the largest eigenvalues. Even for $\V <1$ for the random case ($\V < \frac{1}{4}$
for the GOE case), there is non-zero probability that the spectral edge is greater than $1$.

There are two ways to proceed. One is to compute $\fourmom[\z^*]$ using the definition
involving $\tr[((\Id-\W)^{-\tpose}(\Id-\W)^{-1})^{2}]$. In the limit $\N\to\infty$, this can be solved
numerically using operator-valued free probability theory.
The random and GOE cases are solved for this way in
Appendix \ref{app:op_free_prob}, where we can obtain
exact analytic solutions.

Another approach is to switch the order of limits. We define:
\begin{equation}
\fourmom[\z^*](\infty) \equiv \lim_{L\to\infty}\lim_{\N\to\infty} \expect_{\W}\left[\tr\left[\sum_{j, k, l, m = 0}^{L}(\W^{j})^{\tpose}\W^{k}(\W^{l})^{\tpose}\W^{m}\right]\right]
\end{equation}
In the GOE case, we can write down an integral equation for $\fourmom[\z^*](\infty) $:
\begin{equation}
\fourmom[\z^*](\infty) = \frac{2}{\pi}\int_{-1}^{1} \frac{1}{(1-2\sqrt{V}x)^{4}}\sqrt{1-x^2}dx
\end{equation}
There is a non-analytic formal power series solution. However,
we can understand the behavior for $\V$ near the critical value analytically.
Let $\sqrt{\V} = 1/2-\delta$, for some $\delta\ll 1$. Then, after shifting
the integration variable we have:
\begin{equation}
\fourmom[\z^*](\infty) = \frac{2}{\pi}\int_{0}^{2} \frac{1}{(2\delta+(1-2\delta)x)^4}\sqrt{x(2-x)}  dx
\end{equation}
In the limit of small $\delta$, we have:
\begin{equation}
\fourmom[\z^*](\infty) = \left(\frac{2}{\pi}\int_{0}^{2} \frac{1}{(2\delta+x)^4}\sqrt{x(2-x)}  dx\right)(1+O(\delta))
\end{equation}

Let $y = x/2\delta$. Then we have:
\begin{equation}
\fourmom[\z^*](\infty) \approx \frac{1}{2^{3/2}\pi} \delta^{-2.5}\int_{0}^{1/\delta} \frac{\sqrt{y(2-2\delta y)}}{(1+y)^4}dy
\end{equation}
with relative error $O(\delta)$. As a further approximation, we have:
\begin{equation}
\fourmom[\z^*](\infty) \approx \frac{1}{2\pi} \delta^{-2.5}\int_{0}^{\infty} \frac{\sqrt{y}}{(1+y)^4}dy
\end{equation}
again with relative error $O(\delta)$. In total we have:
\begin{equation}
\fourmom[\z^*](\infty) = \frac{1}{8}\delta^{-2.5}(1+O(\delta))
\end{equation}
for $\delta\ll 1$.

Since the squared second moment only scales as $\delta^{-2}$, for small $\delta$ we have
$\fourmom[\z^*](\infty)= O(\delta^{-2.5})$.

This analysis lets us draw distinctions between the tied and untied cases, as well as between
the matrix ensembles. In all cases, as expected the tied cases display more variance than the untied
cases. These differences even show up asymptotically as $\delta$, the distance to the transition, becomes small.
The orthogonal case has one large advantage compared to the random and GOE cases: for any finite
$\N$, we expect convergence for $\V$ below the critical range. However, for both GOE and
random matrices there is a chance of divergence for finite $\N$, which increases
as $\delta$ decreases.

However, in the intermediate regime where finite-$\N$ effects are small, the analysis of
$\fourmom[\z^*](\infty)$ suggests that the GOE \emph{typically} has less variance than the
orthogonal. The analytical results in Appendix \ref{app:op_free_prob} and the numerical
results in Figure \ref{fig:four_mom} confirm this analysis.

\section{Non-linear DEQ theory}

\label{app:non_lin_deq}

\subsection{DEQs in the infinite-width limit}

\label{app:deq_NTK}

Given a non-linear DEQ given by the implicit equation
\begin{equation}
\z^* = \phi(\W\z^*)+\x
\label{eq:deq_eq_app}
\end{equation}
which is sometimes thought of as an iterative equation of the form
\begin{equation}
\z_{t+1} = \phi(\W\z_{t})+\x
\label{eq:general_deq_app}
\end{equation}
we can analyze the properties of its infinite-width representations.
The case of untied weights ($\W$ replaced by independent $\W_{t}$
in Equation \ref{eq:general_deq_app}) has been previously studied \cite{feng_neural_2020}.

In order to make infninite-width quantities well defined, we equip the DEQ
layer with a readout vector $\v$ in order to construct the scalar function
\begin{equation}
f_{\v, \W}(\x)\equiv \v^{\tpose}\z^*(\x)
\end{equation}
We will omit the subscripts unless necessary.

The NNGP kernel \cite{ lee_wide_2019} of $f$ can be defined as:
\begin{equation}
\mathcal{K}(\x,\x') = \expect_{\v,\W}[f(\x)f(\x')]
\end{equation}
which evaluates to:
\begin{equation}
\mathcal{K}(\x,\x') = \frac{1}{\N}\expect_{\v,\W}[\z^*(\x)\cdot\z^*(\x')]
\end{equation}

In order to compute the NTK, we must compute the derivative
with respect to $\W$:
\begin{equation}
\frac{\partial f}{\partial \W} = \frac{\partial \v^{\tpose}\z^*}{\partial \W}
\end{equation}
Using Equation \ref{eq:deq_eq_app} and the implicit function theorem, we have:
\begin{equation}
\frac{\partial f}{\partial \W} =  (\Id-\phi'(\z^*)\circ\W)^{-\tpose}\v(\phi'(\z^*)\circ\z^*)^{\tpose}
\end{equation}

For fixed $\v$, the NTK $\ntk(\x,\x')$ with respect to $\W$ is given by:
\begin{equation}
\ntk(\x,\x') \equiv \expect\left[\frac{\partial f}{\partial\W}\cdot \frac{\partial f}{\partial\W}\right]
\end{equation}

The NTK, after averaging over $\v$, is therefore:
\begin{equation}
\ntk(\x,\x') = \expect\left[\tr((\Id-\phi'(\z^*)\circ\W)^{-\tpose}(\Id-\phi'((\z')^*)\circ\W)^{-1}) \left(\phi'(\z^*)\circ\z^*\right)\cdot \left(\phi'((\z')^*)\circ(\z')^*\right)\right]
\label{eq:deq_implicit_app}
\end{equation}

The first term in Equation \ref{eq:deq_implicit_app} gives the alignment of the Jacobians, computed using the implicit function theorem.
The second term gives the alignment of the fixed points themselves, mediated by the
derivative of the elementwise non-linearity $\phi$. We expect (but do not prove here) that in many cases
concentration of measure means the two terms are statistically independent.

We see from the form of the equation alone that, much like the linear case,
the statistics of the matrix ensemble which $\W$ is drawn from affect the kernel, beyond simply the first and second moments.
Orthogonal $\W$ should behave identically to random $\W$ in the large $\N$ limit; however other ensembles
like the GOE will give us very different NTKs.

\section{Free probability calculations}

\label{app:free_prob}

\subsection{Spectrum of $\m{D} \W$}

\label{app:diag_W}

Let $\m{D}$ be a diagonal matrix, and $\W$ be a GOE matrix. If $\m{D}$ is freely-independent of $\W$, and is non-negative,
then we can use free probability theory to solve for the spectrum of $\M = \m{D}\W$.
We first note that $\m{D}\W$ has the same spectrum as $\m{D}^{1/2}\W\m{D}^{1/2}$. This means that $\m{D}\m{W}$ has
real eigenvalues. We can use the Stieltjes transform of $\m{D}\W$ to solve for its spectrum using
the theory of free multiplicative convolutions \cite{mingo_free_2017}.

We review the basics of the theory here.
We recall that the Stieltjes transform is given by:
\begin{equation}
G_{\M}(z) \equiv \tr\left[(z-\M)^{-1}\right]
\end{equation}
where $\tr$ is the normalized trace. The spectrum can be recovered from the Stieltjes
transform via the relation
\begin{equation}
\rho(x) = -\frac{1}{\pi}\lim_{\epsilon\to0^{+}} \Im[G_{\M}(x+i\epsilon)]
\end{equation}
The Stieltjes transform is related to the moment generating function $M$ by
\begin{equation}
M_{\M}(z) \equiv \sum_{k=1}^{\infty} m_{k}z^{-k} = zG_{\M}(z)-1 
\end{equation}
where $m_k = \tr[\M^k]$.
Since $\W$ and $\DD$ are freely independent with real spectra, we can compute the spectrum of their product
using the S-transform - since the S-transform of a product of freely independent variables is the product
of the S-transform. In terms of the moment generating function, we have:
\begin{equation}
S_{\M}(z) = \frac{1+z}{zM^{-1}_{\M}(z)}
\end{equation}
where $M^{-1}$ is the functional inverse of the moment generating function. Given the S-transform,
the MGF can be recovered using
\begin{equation}
M_{\M}^{-1}(z) = \frac{1+z}{z S_{\M}(z)}
\end{equation}

We begin by computing the S-transforms of $\DD$ and $\W$. If $\W$ is a standard GOE
element, we have:
\begin{equation}
G_{\W}(z) = \frac{z-\sqrt{z^2-4}}{2}
\end{equation}
We use the branch of the square root function which has a branch cut on $[-2, 2]$, which corresponds to
a cut on $(-\infty, 0]$ for the input of the square root. We choose the branch that has negative
imaginary part just above the real axis.

This gives us a moment generating function of
\begin{equation}
M_{\W}(z) = \frac{z^2-z\sqrt{z^2-4}}{2}-1
\end{equation}
Let $w = M_{\W}^{-1}(z)$, the inverse MGF. We have:
\begin{equation}
2(z+1) = w^2-w\sqrt{w^2-4}
\end{equation}
\begin{equation}
2(z+1)-w^2 = -w\sqrt{w^2-4}
\end{equation}
\begin{equation}
4(z+1)^2-4(z+1)w^2 +w^4= w^2(w^2-4)
\end{equation}
\begin{equation}
4(z+1)^2-4z w^2 = 0
\end{equation}
This gives us an inverse moment generating function of
\begin{equation}
M_{\W}^{-1}(z) = \sqrt{\frac{(z+1)^2}{z}}
\end{equation}
which can also be written as
\begin{equation}
M_{\W}^{-1}(z) = \sqrt{z+\frac{1}{z}}
\end{equation}
This function has a branch cut from $-1$ to $+\infty$ (corresponding to $(-\infty, 0]$ for the argument of the square root),
and takes on positive values for real negative $z$.
This gives us the S-transform:
\begin{equation}
S_{\W}(z) = \frac{1+z}{z}\left[z+\frac{1}{z}\right]^{-1/2}
\end{equation}
Therefore we have:
\begin{equation}
S_{\M}(z) = S_{\W}(z)S_{\m{D}}(z)
\end{equation}

When $\phi$ is the hart-tanh function, we can solve for the spectrum analytically.
This means that
the elements of $\DD$ are Bernoulli random variables, with a probability $p(h^*)$ of being
$1$ which can be computed using the statistics of $h^*$.
The Stieltjes transform is
\begin{equation}
G_{\DD}(z) = \frac{1-p(h^*)}{z}+\frac{p(h^*)}{z-1}
\end{equation}
The moment generating function is therefore
\begin{equation}
M_{\DD}(z) = \sum_{k=1}^{\infty} p(h^*) z^{-k} = \frac{p(h^*)}{z}\left(1-1/z\right)^{-1} = \frac{p(h^*)}{z-1}
\end{equation}
The inverse is given by
\begin{equation}
M_{\DD}^{-1}(z) = \frac{p(h^*)}{z}+1 = \frac{p(h^*)+z}{z}
\end{equation}
The S-transform is given by
\begin{equation}
S_{\DD}(z) = \frac{1+z}{z+p(h^*)}
\end{equation}

Therefore, the overall S-transform is given by:
\begin{equation}
S_{\M}(z) = \frac{(1+z)^2}{z(z+p(h^*))}\left[z+\frac{1}{z}\right]^{-1/2}
\end{equation}
The inverse moment generating function is given by
\begin{equation}
M^{-1}_{\M}(z) = \frac{z+p(h^*)}{z+1}\left[z+\frac{1}{z}\right]^{1/2}
\end{equation}
which simplifies to
\begin{equation}
M^{-1}_{\M}(z) = (z+p(h^*))\left[z\right]^{-1/2}
\end{equation}
The moment generating function obeys the equation
\begin{equation}
z^2M = (M+p(h^*))^2
\end{equation}
\begin{equation}
 (M+p(h^*))^2 -z^2M= 0
\end{equation}
\begin{equation}
M^2+(2p(h^*)-z^2)M+p(h^*)^2 = 0
\end{equation}
Solving for the MGF, we have:
\begin{equation}
M_{\M}(z) = \frac{-(2p(h^*)-z^2)\pm\sqrt{(2p(h^*)-z^2)^2-4p(h^*)^2}}{2}
\end{equation}
Simplification gives us:
\begin{equation}
M_{\M}(z) = -p(h^*)-\frac{z^2\pm z\sqrt{z^2-4p(h^*)}}{2}
\end{equation}
We expect the first moment to vanish and the second to be positive which gives us
\begin{equation}
M_{\M}(z) = -p(h^*)-\frac{z^2- z\sqrt{z^2-4p(h^*)}}{2}
\end{equation}

The Stieltjes transform is therefore given by
\begin{equation}
G_{\M}(z) = \frac{1-p(h^*)}{z}-\frac{z- \sqrt{z^2-4p(h^*)}}{2}
\end{equation}
Therefore, the spectrum of $\M$ is given by a combination of a delta function at $0$ of weight
$1-p(h^*)$ and a semi-circular law with radius $2\sqrt{p(h^*)}$.

\subsection{Operator valued free probability calculation of $\fourmom$}

\label{app:op_free_prob}

In this section we will compute the spectrum of $(\Id-\W)^{-\tpose}(\Id-\W)^{-1}$ for GOE and random
$\W$ using operator valued
free probability. In particular, we're interested in recovering the $2$nd moment of the spectrum
as $\V$ goes to $1$.

The central object in operator-valued free probability theory is the \emph{operator valued} Stieltjes transform.
Let $\hat{\m{P}}$
be a fixed $k\times k$ block matrix with $\N\times\N$ blocks. We define the function
$\G:\mathbb{C}^{k\times k}\to\mathbb{C}^{k\times k}$
as
\begin{equation}
\G(\B) = \condtr[ (\B-\hat{\m{P}})^{-1}]
\end{equation}
Here the $\condtr$ operator applies the normalized trace to each $\N\times\N$ sub-block of the input.

As we will see, we can often compute the ordinary Stieltjes transform of rational function $f(\W)$ by choosing some
$\hat{\m{P}}$ which is linear in $\W$ and $\W^{\tpose}$,
such that for the appropriate choice of $\B$,
$G_{f(\W)}(z)$ is the first element of $\G(\B)$. In this approach $\hat{\m{P}}$ is often referred to as a
\emph{linear pencil}. We can then use well-developed techniques derived from free convolution theory to
compute $\G(\B)$.
For a more detailed description of the techniques, we refer the reader to \cite{mingo_free_2017}.

\subsubsection{Warmup: semicircular case}

For a GOE matrix $\W$, the matrix $(\Id-\W)^{-1}$ is already symmetric. Therefore its Stieltjes transform is
well defined and analytic, and its fourth moment gives us the trace $\tr[((\Id-\W)^{-\tpose}(\Id-\W)^{-1})^{2}]$ needed
to compute $\fourmom$.

Consider the block matrix $\sm{\Lambda}$ given by
\begin{equation}
\sm{\Lambda} = \begin{pmatrix}
z & 0 \\
0 & 0 
\end{pmatrix}
\end{equation}
where $z$ is proportional to the $\N\times\N$
dimensional identity matrix $\Id$. Then, if $\hat{\m{P}}$ has the structure
\begin{equation}
\hat{\m{P}} = \begin{pmatrix}
0 & \m{V}\\
\m{Q} & \m{P}
\end{pmatrix}
\end{equation}
where $0$ is the first $\N\times\N$ block, the operator valued Stieltjes transform has the following property:
\begin{equation}
\G(\sm{\Lambda})_{11} = \tr[z+\m{V}\m{P}^{-1}\m{Q}] = G_{-\m{V}\m{P}^{-1}\m{Q}}(z)
\end{equation}

With judicious choice of $\hat{\m{P}}$, we can compute the Stieltjes transform of $(\Id-\W)^{-1}$. Consider
\begin{equation}
\hat{\m{P}} = \begin{pmatrix}
0 & \Id\\
\Id & \W-\Id
\end{pmatrix}
\end{equation}
Then, $\G(\sm{\Lambda})_{11} = G_{(\Id-\W)^{-1}}(z)$.

If $\hat{\m{P}}$ can be decomposed into a sum $\m{\hat{M}}$ + $\m{\hat{F}}$, where $\m{\hat{M}}$
is constant, and $\m{\hat{F}}$ is a linear sum of
(asymptotically) freely independent matrices with $0$ trace, then we can solve for $\G(\B)$. Define
$\Z = \B-\m{\hat{\M}}$. Then we have:
\begin{equation}
\Z \G = \Id +\eta(\G)\G
\label{eq:self_consistent}
\end{equation}
where the covariance function $\eta: \mathbb{C}^{k\times k}\to\mathbb{C}^{k\times k}$ is defined as
\begin{equation}
\eta(\G)_{ij} = \sum_{kl} \covtr(i,k; l, j) \G_{kl}
\end{equation}
where
\begin{equation}
\covtr(i,k; l, j) \equiv \tr[(\hat{\m{F}})_{ik}(\hat{\m{F}})_{lj}]
\end{equation}

In this case, we have:
\begin{equation}
\eta(\G) = \begin{pmatrix}
0 & 0 \\
0 & \V G_{22}
\end{pmatrix}
\end{equation}

For $\B = \sm{\Lambda}$, we have
\begin{equation}
\Z = \begin{pmatrix}
z & -\Id \\
-\Id & \Id
\end{pmatrix}
\end{equation}
which gives the matrix equation:
\begin{equation}
\begin{pmatrix}
z G_{11}-G_{21} & z G_{12}-G_{22}\\
-G_{11}+G_{21} & -G_{12} +G_{22}
\end{pmatrix} = 
\begin{pmatrix}
1 & 0\\
\V G_{21}G_{22} & 1+\V G_{22}^2
\end{pmatrix}
\end{equation}
We can now solve the system of equations for $G_{11}$.

We first note that, since $\sm{\Lambda}$ and $\hat{\m{P}}$
are symmetric, so is $\G$. Therefore, $G_{12} = G_{21}$.
The first equation gives us
\begin{equation}
z G_{11} - 1 = G_{21} = G_{12}
\end{equation}
We recall that $z G_{11}-1 = M$, where $M$ is the
moment generating function of $(\Id-\W)^{-1}$
(in $z^{-1}$). We will solve for $M$ since the coefficient of the fourth order term gives us
$\fourmom$.
The second equation gives us
\begin{equation}
zG_{12} = z M = G_{22}
\end{equation}
The fourth equation gives us
\begin{equation}
-M+zM = 1+\V z^2 M^2
\end{equation}
Solving for $M$, we have
\begin{equation}
M(z) = \frac{(z-1)\pm\sqrt{(z-1)^2-4\V z^2}}{2\V z^2}
\end{equation}
We can choose the correct root by evaluating for small
$\V$. We know that
\begin{equation}
\lim_{\V\to 0} M(z) = \frac{1}{z-1}
\end{equation}
In particular this expansion is valid around large $z$.
This means that the negative root is the correct one and we have
\begin{equation}
M(z) = \frac{(z-1)-\sqrt{(z-1)^2-4\V z^2}}{2\V z^2}
\end{equation}

We can compute $\tr[((\Id-\W)^{-\tpose}(\Id-\W)^{-1})^{2}]$ by taking derivatives of $M(z)$. Let $w = z^{-1}$.
Then we have:
\begin{equation}
\tr[((\Id-\W)^{-\tpose}(\Id-\W)^{-1})^{2}] = \frac{1}{4!}\left.\frac{d^{4} M}{dw^4}\right|_{w = 0}
\end{equation}
Writing $M(w)$, we have
\begin{equation}
M(w) = w^{2}\frac{(w^{-1}-1)-\sqrt{(1-w^{-1})^2-4\V w^{-2}}}{2\V}
\end{equation}
For ease of computation, we define
\begin{equation}
g(x) = \frac{-x-\sqrt{x^2-4\V}}{2\V}
\end{equation}
Then we have:
\begin{equation}
M(w) = w g(w-1)
\end{equation}
This gives us:
\begin{equation}
\frac{d^4 M}{dw^4} = 4\frac{d^3}{dw^3} g(w-1)+w\frac{d^4}{dw^4} g(w-1)
\end{equation}
At $w = 0$, this evaluates to:
\begin{equation}
\left.\frac{d^4 M}{dw^4}\right|_{w = 0} = \left.\frac{2}{\V}\left(\frac{3 w}{(w^2-4\V)^{3/2}}-\frac{3 w^3}{(w^2-4\V)^{5/2}}\right)\right|_{w = -1}
\end{equation}
which gives us
\begin{equation}
\tr[((\Id-\W)^{-\tpose}(\Id-\W)^{-1})^{2}] = \frac{1}{4\V}\left( \frac{1}{(1-4\V)^{5/2}}-\frac{1}{(1-4\V)^{3/2}} \right)
\end{equation}
for GOE distributed $\W$.

\subsubsection{Random $\W$ case}

\label{app:rand_inv}

The case of random $\W$ is more complicated, because the decomposition in Equation \ref{eq:self_consistent}
depends on the constituent elements being semi-circular.
We therefore first decompose $\W$ and $\W^{\tpose}$ into two freely-independent self-adjoint matrices:
\begin{equation}
\X\equiv \frac{\W+\W^{\tpose}}{2},~\Y\equiv \frac{\W-\W^{\tpose}}{2i}
\end{equation}
We can recover $\W$ via
\begin{equation}
\X+i\Y = \W, ~ \X-i\Y = \W^{\tpose}
\end{equation}
With this decomposition, we can begin to take advantage of operator-valued
free probability to write the Stieltjes transform of $(\Id-\W)^{-\tpose}(\Id-\W)^{-1}$ as
part of a larger matrix.

We construct the linear pencil $\hat{\m{P}}$ as follows.
Cnsider the block matrix
\begin{equation}
\m{P} = 
\begin{pmatrix}
\Id-2\X & \X+i\Y\\
\X-i\Y & -\Id
\end{pmatrix}
\end{equation}

Computing $\m{P}^{-1}$, we have:
\begin{equation}
(\m{P}^{-1})_{11} = \left( \Id-2\X+(\X+i\Y)(\X-i\Y) \right)^{-1} = \left( \Id-2\X+(\X+i\Y)(\X-i\Y) \right)^{-1} 
\end{equation}
\begin{equation}
(\m{P}^{-1})_{11} = (\Id-\W-\W^{\tpose}+\W\W^{\tpose})^{-1} = (\Id-\W)^{-\tpose}(\Id-\W)^{-1}
\end{equation}

This gives us
\begin{equation}
\hat{\m{P}} = \begin{pmatrix}
0 & \Id & 0\\
\Id & -\Id+2\X & -\X-i\Y \\
0 & -\X+i\Y & \Id
\end{pmatrix}
\end{equation}
As desired, $\hat{\m{P}}^{-1}_{11} = (\Id-\W)^{-\tpose}(\Id-\W)^{-1}$.

Repeating the setup of the operator-valued Stieltjes transform (this time for $3\times 3$ complex matrices
rather than $2\times 2$, with the equivalent definition of $\sm{\Lambda}$, we have

\begin{equation}
\Z = \begin{pmatrix}
z & -\Id & 0\\
-\Id& \Id & 0 \\
0 & 0 & -\Id
\end{pmatrix}
\end{equation}
for a complex $z$. We also define:
\begin{equation}
\tilde{\X} = \begin{pmatrix}
0 & 0 & 0\\
0 & 2\X & -\X \\
0 & -\X & 0
\end{pmatrix}, ~\tilde{\Y} = \begin{pmatrix}
0 & 0 & 0\\
0 & 0 & -i\Y \\
0 & i\Y & 0
\end{pmatrix},
\end{equation}
So we have
\begin{equation}
\sm{\Lambda}-\hat{\m{P}} = \Z-\tilde{\X}-\tilde{\Y} = 
\begin{pmatrix}
z & -\Id & 0\\
-\Id& \Id-2\X & \X+i\Y \\
0 & \X-i\Y & -\Id
\end{pmatrix}
\end{equation}

Equation \ref{eq:self_consistent} holds, where now the covariance map is defined by
\begin{equation}
\covtr(i,k; l, j) \equiv \tr[(\tilde{\X}+\tilde{\Y})_{ik}(\tilde{\X}+\tilde{\Y})_{lj}]
\end{equation}

We can compute the individual covariance terms directly. Define $\sm{\covtr}_{ij} = \covtr(i,j;i,j)$.
Then we have:
\begin{equation}
\sm{\covtr} = \begin{pmatrix}
0 & 0 & 0 \\
0 & 2\V& 0\\
0 & 0 & 0
\end{pmatrix}
\end{equation}
We note that $\covtr(1j;lm) = \covtr(i1;lm) = \covtr(33; lm) = 0$ and
\begin{equation}
\covtr(22;23) = \covtr(23;22) = -\V
\end{equation}
Finally, $\covtr(23,32) = \covtr(32,23) = \V$.

We can now write out a system of $9$ equations that the Stieltjes transform obeys. We have:
\begin{equation}
\Z\G = \begin{pmatrix}
zG_{11}-G_{21}  & zG_{12}-G_{22} & zG_{13}-G_{23} \\
-G_{11}+G_{21}  & -G_{12}+G_{22}  & -G_{13}+G_{23}  \\
-G_{31}  &-G_{32} & -G_{33}
\end{pmatrix}
\end{equation}
as well as
\begin{equation}
\eta(\G) = \V\begin{pmatrix}
0 & 0 & 0 \\
0 & 2G_{22} -(G_{23}+G_{32})+G_{33} & -G_{22}\\
0 &  -G_{22} & G_{22}
\end{pmatrix}
\end{equation}
which gives us
\begin{equation}
(\eta(\G)\G)_{1j} = 0
\end{equation}
\begin{equation}
(\eta(\G)\G)_{2j} = \V\left([2G_{22}-(G_{23}+G_{32})+G_{33}]G_{2j}-G_{22}G_{3j}\right)
\end{equation}
\begin{equation}
(\eta(\G)\G)_{3j} = \V G_{22}(-G_{2j}+G_{3j})
\end{equation}

The top row of equations gives us:
\begin{equation}
z G_{11}-G_{21} = 1
\end{equation}
\begin{equation}
z G_{12}-G_{22} = 0
\end{equation}
\begin{equation}
z G_{13}-G_{23} = 0
\end{equation}

The middle row is:
\begin{equation}
-G_{11}+G_{21} = \V\left([2G_{22}-(G_{23}+G_{32})+G_{33}]G_{21}-G_{22}G_{31}\right)
\end{equation}
\begin{equation}
-G_{12}+G_{22} = 1+\V\left([2G_{22}-(G_{23}+G_{32})+G_{33}]G_{22}-G_{22}G_{32}\right)
\end{equation}
\begin{equation}
-G_{13}+G_{23} = \V\left([2G_{22}-(G_{23}+G_{32})+G_{33}]G_{23}-G_{22}G_{33}\right)
\end{equation}

The bottom row gives us:
\begin{equation}
-G_{31} = \V G_{22}(-G_{21}+G_{31})
\end{equation}
\begin{equation}
-G_{32} =  \V G_{22}(-G_{22}+G_{32})\end{equation}
\begin{equation}
-G_{33} = 1+ \V G_{22}(-G_{23}+G_{33})
\end{equation}

We note that $\sm{\Lambda}-\hat{\m{P}}$ is
symmetric. Therefore $\G(\sm{\Lambda})$
is as well, and we can re-write the first set of equations as
\begin{equation}
z G_{11}-G_{12} = 1
\end{equation}
\begin{equation}
z G_{12}-G_{22} = 0
\end{equation}
\begin{equation}
z G_{13}-G_{23} = 0
\end{equation}
In the third set, the first two
equations are now redundant. Therefore we have
\begin{equation}
-G_{23} =  \V G_{22}(-G_{22}+G_{23})
\end{equation}
\begin{equation}
-G_{33} = 1+ \V G_{22}(-G_{23}+G_{33})
\end{equation}
From the second set, we have
\begin{equation}
-G_{12}+G_{22} = 1+\V\left([2G_{22}-3G_{23}+G_{33}]G_{22}\right)
\end{equation}

We can now attempt to solve for
$G_{11}$.
Using the first two equations, we have
\begin{equation}
G_{12} = zG_{11}-1
\end{equation}
\begin{equation}
G_{22} = z(zG_{11}-1)
\end{equation}

Using the fifth equation, we have
\begin{equation}
G_{33} = \frac{\V G_{22}G_{23}-1}{1+\V G_{22}}
\end{equation}
Substituting into the last equation, we have
\begin{equation}
-G_{12}+G_{22} = 1+\V\left(\left[2G_{22}-3G_{23}+\frac{\V G_{22}G_{23}-1}{1+\V G_{22}}\right]G_{22}\right)
\end{equation}

The fourth equation gives us
\begin{equation}
G_{23} = \frac{\V G_{22}^{2}}{1+\V G_{22}}
\end{equation}
Substitution gives us
\begin{equation}
-G_{12}+G_{22} = 1+\V\left(\left[2G_{22}-\frac{3\V G_{22}^{2}}{1+\V G_{22}}+\frac{\V^2 G_{22}^{3}}{(1+\V G_{22})^2}  -\frac{1}{1+\V G_{22}}\right]G_{22}\right)
\end{equation}
We define $\tG \equiv G_{12} = zG_{11}-1$
(the moment generating function). Then we have:
\begin{equation}
-(1-z)\tG = 1+\V\left(\left[2z\tG-\frac{3\V (z\tG)^{2}}{1+\V z\tG}+\frac{\V^2 (z\tG)^{3}}{(1+\V z \tG)^2}  -\frac{1}{1+\V z\tG}\right]z\tG\right)
\end{equation}
We get the following cubic equation
for $\tG$:
\begin{equation}
\V^2 z^2\tG^3 +2\V z \tG^2 + ((\V-1)z+1)\tG+1 = 0
\label{eq:cubic_G}
\end{equation}

The cubic can be solved for analytically
using Cardano's formula to obtain the moment generating function. However, for the
purposes of computing $\fourmom$ we simply
need to compute $m_{2}$, where $\tG$ has the
expansion
\begin{equation}
\tG(z) = \sum_{k=1}^{\infty} m_k z^{-k}
\end{equation}
The cubic in Equation \ref{eq:cubic_G} defines a second-order recursion for the $m_k$. We know that,
by definition, there is no $0$th order term in $M(z)$. The $0$th order term in Equation \ref{eq:cubic_G} gives us $m_1 = \frac{1}{1-\V}$
(as expected). The first order term (coefficient of
$z^{-1}$) gives us
\begin{equation}
\V^2 m_1^3+2\V m_1^2 + (\V-1)m_2+m_1 = 0
\end{equation}
Solving for $m_2$, we have:
\begin{equation}
m_2 = \frac{\V^2}{(1-\V)^4}+\frac{2\V}{(1-\V)^3}+\frac{1}{(1-\V)^2}
\end{equation}
which matches up well with empirical measurements until $1-\V$ becomes small,
when finite size effects begin to matter (Figure \ref{fig:four_mom}).

\section{Experimental setup}

\label{app:experiments}

We based our experiments on a Haiku implementation of a DEQ transformer \cite{khan_deq_2020} which uses the
basic transformer layer \cite{al-rfou_characterlevel_2019}.
We used a pre-trained sentencepiece tokenizer
 \cite{kudo_sentencepiece_2018}. We trained on
Wikitext-103 \cite{merity_pointer_2016} with
a batch size of 512 and a context length of
128.

We added our own code to initialize the dense
layer and the attention layers with different matrix families, though our experiments only 
modified the dense layers. Our hidden state had dimension $700$. In the orthogonal and random
cases, the dense layer expanded the state to
dimension $2800$ before projecting back down to $700$. In the GOE case, since the matrices must 
be square, we kept the dimension at $700$ throughout the network.
We ran for $20$ steps of the Broyden solver.

For the experiments with multiple seeds, we used a learning rate of $10^{-3}$ with a linear
warmup for $2\cdot 10^{3}$ steps, followed by a cosine learning rate decay for $5\cdot 10^{4}$
steps. These parameters were chosen for their overall good performance on the random
initialization with $\sqrt{\V} = 0.1$. We found
that increasing or decreasing the learning rate did not significantly improve performance in any
case.


\end{document}